\documentclass[11pt]{article}

\usepackage[round]{natbib}

\usepackage{amsthm,amsmath,amsfonts}
\usepackage{subcaption}
\usepackage{mdframed}

\usepackage{dsfont}
\usepackage{algorithm}
\usepackage{algorithmic}
\usepackage{hyperref}
\usepackage{mathtools}
\usepackage{tikz}
\usepackage{booktabs}
\usepackage{multirow}
\usepackage{soul}

\usepackage[margin=1.25in]{geometry}
\usepackage{authblk}

\usepackage[draft]{fixme}
\fxsetup{inline,nomargin,theme=color}

\def\sumphantom{\vphantom{\sum_i^1}}

\def\E{\mathbb{E}}

\def\ipm{\textnormal{IPM}}
\def\mipm{\mbox{\emph{IPM}}}

\def\cC{\mathcal{C}}
\def\cD{\mathcal{D}}
\def\cE{\mathcal{E}}
\def\cF{\mathcal{F}}
\def\cG{\mathcal{G}}
\def\cH{\mathcal{H}}

\def\cL{\mathcal{L}}

\def\cO{\mathcal{O}}

\def\cR{\mathcal{R}}
\def\cS{\mathcal{S}}
\def\cT{\mathcal{T}}

\def\cX{\mathcal{X}}
\def\cY{\mathcal{Y}}
\def\cZ{\mathcal{Z}}

\def\bbE{\mathbb{E}}

\def\bbR{\mathbb{R}}

\def\htau{\hat{\tau}}

\def\hp{\hat{p}}
\def\hR{\hat{R}}

\newcommand\indepm{\protect\mathpalette{\protect\independenT}{\perp}}
\def\independenT#1#2{\mathrel{\rlap{$#1#2$}\mkern2mu{#1#2}}}
\DeclareMathOperator*{\indep}{\indepm}

\newtheorem{thmlem}{Lemma}
\newtheorem{thmcol}{Corollary}

\newtheorem{thmthm}{Theorem}
\newtheorem{thmasmp}{Assumption}

\theoremstyle{definition}
\newtheorem{thmrem}{Remark}
\newtheorem{thmdef}{Definition}

\makeatletter
\newtheorem*{rep@theorem}{\rep@title}
\newcommand{\newreptheorem}[2]{%
\newenvironment{rep#1}[1]{%
 \def\rep@title{#2 \ref{##1} (Restated)}%
 \begin{rep@theorem}}%
 {\end{rep@theorem}}}
\makeatother

\newreptheorem{theorem}{Theorem}
\newreptheorem{lemma}{Lemma}
\newreptheorem{col}{Corollary}
\newreptheorem{def}{Definition}

\DeclareMathOperator*{\argmin}{arg\,min}

\newcommand\encircle[1]{%
  \tikz[baseline=(X.base)]
    \node (X) [draw, shape=circle, inner sep=0] {\strut #1};}

\newcommand{\edit}{}
\newcommand{\blockedit}{}

\title{\LARGE Generalization Bounds and Representation Learning for Estimation of Potential Outcomes and Causal Effects}

\date{\vspace{-3em}}

\author[1]{Fredrik D.~Johansson\thanks{Correspondence to: fredrik.johansson@chalmers.se}}
\author[2]{Uri Shalit}
\author[3]{Nathan Kallus}
\author[4]{\\David Sontag}%
\affil[1]{Chalmers University of Technology}%
\affil[2]{Technion, Israel Institute of Technology}%
\affil[3]{Cornell Tech}%
\affil[4]{Massachusetts Institute of Technology}%

\begin{document}

\maketitle

\begin{abstract}%
Practitioners in diverse fields such as healthcare, economics and education are eager to apply machine learning to improve decision making. The cost and impracticality of performing experiments and a recent monumental increase in electronic record keeping has brought attention to the problem of evaluating decisions based on non-experimental observational data. This is the setting of this work. In particular, we study estimation of individual-level \edit{potential outcomes and} causal effects---such as a single patient's response to alternative medication---from recorded contexts, decisions and outcomes. We give generalization bounds on the error in estimated \edit{outcomes} based on \edit{distributional} distance measures between \edit{re-weighted samples of groups receiving different treatments}. We provide conditions under which our bound is tight and show how it relates to results for unsupervised domain adaptation. Led by our theoretical results, we devise \edit{algorithms which learn representations and weighting functions} that minimize our bound by regularizing the representation's induced treatment group distance, and encourage sharing of information between treatment groups. Finally, an experimental evaluation on real and synthetic data shows the value of our proposed representation architecture and regularization scheme.
\end{abstract}

%
%

\section{Introduction}
\label{sec:introduction}
Evaluating intervention decisions is a key question in many diverse fields including medicine, economics, and education. In medicine, an optimal choice of treatment for a patient in the intensive care unit may mean the difference between life and death. In public policy, job reforms have impact on the unemployment rate and the economy of a nation. To evaluate such interventions,  we must study their \emph{causal effect}---the difference in an outcome of interest under alternative choices of intervention. Since only one option may be carried out at a time, any data to support such evaluations only reveals the outcome of the action taken and never the outcome of the action not taken, which remains an unknown \emph{counterfactual}. To estimate causal effects, we must therefore infer what would have happened had we made another decision, \edit{predicting the \emph{potential outcomes} \citep{rubin2005causal} of unexplored interventions}. Furthermore, to decide on personalized interventions, such as tailoring treatments to patients, we must understand individual-level causal effects, conditioned on the available information on an individual recorded prior to intervention.

In this work, we study estimation of individual-level \edit{potential outcomes and} causal effects from non-experimental, \emph{observational} data \edit{under a machine learning, or risk minimization, perspective}. An observational dataset consists of historical records of interventions, the contexts in which they were made, and the observed outcomes. Our running example is that of patients represented by their medical history, the medication they were prescribed and the outcome of treatment, such as mortality. An individual-level effect measures the causal effect of medication choice, conditioned on what is known about the patient. Finally, though we know which interventions took place, the policy by which interventions were chosen in this data is typically unknown to us.

Working with observational data is our best bet when experiments such as randomized controlled trials (RCT) are infeasible, impractical or prohibitively expensive. While cheaper and easier to implement, observational studies come with new, fundamental difficulties.
Perhaps the most challenging of these is \emph{confounding}---influence of variables that are causal of both the intervention and the outcome, and may introduce spurious, non-causal correlations between the two~\citep{pearl2009causality}.
For example, richer patients might both have more access to certain medications and have better outcomes regardless of medication, making such medications appear better than they might be. Similarly, job training might only be given to those motivated enough to seek it. Na\"ive estimates of causal effects may therefore be biased by subsuming the effect of confounding variables on the outcome.
Here, we make the common assumption that confounding variables, such as wealth or motivation in the examples above, have been measured and can be adjusted for in our estimation. This, however, introduces another difficulty, which is contending with the systematic differences in such variables between different treatment groups.
Moreover, if these groups only partially overlap in terms of variables causal of the outcome, consistent estimation of causal effects (estimates that converge asymptotically to the true effect) may not always be guaranteed.

Causal estimation from observational data has been studied extensively in the statistics, econometrics, and computer science literature, \edit{often focused} on \emph{average} effects in a population or on simple models of effect heterogeneity. With sights set on personalization based on rich data, more flexible models are required, and machine learning is more often considered for the task. When we can no longer make strong assumptions such as linearity or low dimensionality, new questions arise: How well will our models generalize? \edit{What should be our criteria for fitting them?} How should we regularize them? What assumptions are necessary for good performance guarantees or asymptotic consistency? What can be said when these assumptions are not met? With this paper, we begin to answer these questions.

We study estimation of both (i) potential outcomes under interventions and (ii) conditional average treatment effects (CATE), by risk minimization \edit{with flexible model classes}. We show that these problems involve generalization of predictions under distributional shift and how such predictions can be improved using sample re-weighting. \edit{In doing so, we draw connections to unsupervised domain adaptation and show that one solution to CATE estimation is to solve two (possibly dependent) domain adaptation problems. In contrast to most theoretical results in observational causal estimation, bounds obtained using the risk minimization approach are applicable also in finite samples and for misspecified classes of hypotheses}. We use distributional measures of distance between treatment groups to give upper bounds on the marginal risk of hypotheses in a given class---risk caused by model bias, sub-optimal sample weighting or lack of treatment group overlap. 

%

\edit{Sample weighting lies at the heart of many classical approaches to causal estimation, particularly using the \emph{propensity score}---the probability for a subject to receive treatment under the observed policy, given their characteristics. In practice, the propensity score is typically unknown\footnote{This has some notable exceptions such as in advertising, in which an existing policy for serving ads was designed and known to the advertiser~\citep{swaminathan2015counterfactual,lefortier2016large}.} and replacing it with an estimate is prone to introduce bias. However, biased but well-chosen weights may help control variance in small samples~\citep{swaminathan2015counterfactual,hirano2003efficient}. Our bounds make this trade-off explicit through the measure of treatment group distance: the more uniform (and possibly more biased) the weights, the larger the treatment group distance, but the smaller the variance.}

\edit{Our theoretical results are then extended to give generalization bounds for learned invertible representations of the input space, representations in which the distributional distance may be smaller than the original space; see illustration in Figure~\ref{fig:rep}. Combining representation learning with re-weighing allows both for tight bounds when the sample size is large and treatment groups overlap, making the treatment group density ratio an admissible weight, and for a principled bias-variance trade-off for smaller samples. Notably, these guarantees are valid under appropriate assumptions also when treatment and control groups do not fully overlap.} In line with our theoretical results, we devise representation learning algorithms that minimize \edit{the weighted empirical risk of predicted potential outcomes, regularized by the} distributional distance between treatment groups in the representation space. We give conditions under which our algorithms are consistent estimators of the causal effect. Finally, we evaluate our framework on synthetic and real-world benchmark tasks and demonstrate the value of representation learning, sample weighting and our proposed regularization scheme.

Parts of this work have been published in conference proceedings~\citep{johansson2016learning,shalit2016estimating,johansson2019support}. This manuscript extends these works considerably, foremost by developing  generalization bounds and accompanying algorithms that support re-weighting of the relevant risk and distributional distance. This is a significant conceptual change since it allows to give conditions under which our algorithms lead to consistent estimation. In addition, we provide a longer, more self-contained theoretical exposition and compare it to both older results and recent developments.

%
%
\subsection*{Notation and terminology}
\label{sec:notation}
Random variables are denoted with capital roman letters $A, B, C, \dots$ and observations thereof with a corresponding indexed lower-case letter $a_i, b_i, c_i,\dots$. The empirical density of a draw of $m$ samples from a density $p$ is denoted $\hat{p}^m$. Unless stated otherwise, all random variables are distributed according to a fixed distribution $p(A, B, ...)$. Expectations over a variable $X$ distributed according to $p(X)$ are denoted $\E_X[\cdot]$, and conditional expectations over $X$ given $Y$ distributed according to $p(X\mid Y)$, $\E_{X\mid Y}[\; \cdot\mid Y=t]$. When expectations are defined w.r.t. a density $q$, different from $p$, the notation $\E_{X\sim q}[\cdot]$ is used.

%
%

\section{\edit{Potential outcomes and effects of interventions}}
\label{sec:background}
\edit{We introduce the problem of estimating potential outcomes and effects of interventions from observational data.} Throughout the paper, we adopt the running example of estimating the effect of a medical treatment on a patient. This informs our choice of terminology and notation and serves to give intuition for the mathematical quantities involved. However, the applicability of the theory and algorithms described are in no way limited to this domain.

Using the Neyman-Rubin potential outcome framework \citep{imbens2015causal}, we associate each unit (e.g., patient) $i=1,\dots,m$ with the following \edit{partially observed} random variables, illustrated by samples in Figure~\ref{fig:data}.
\begin{figure}
    \centering
    \includegraphics[width=.95\textwidth]{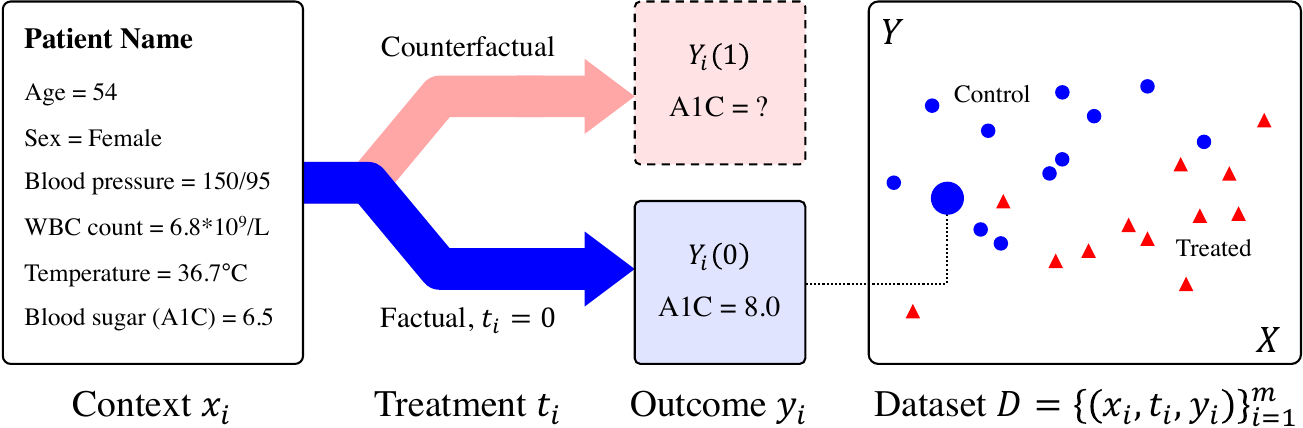}
    \caption{{\blockedit Formation of an observational dataset of patient contexts $X$, treatments $T$ and outcomes $Y$ in an imagined application to diabetes management. Treatments are assigned based on the context and determines which of the two potential outcomes, $Y(0)$ and $Y(1)$, is observed for the patient, here in terms of blood sugar (A1C). The observations are then recorded in the dataset $D$. The multivariate context $X$ is represented in one dimension for illustration purposes. In the figure on the right, red and blue colors denote which treatment the patient received.}}
    \label{fig:data}
\end{figure}
\begin{itemize}
\item A \emph{context} $X_i\in \cX \subseteq \mathbb{R}^d$, defined by all information observed about a patient before the choice of treatment is determined. These covariates may
influence both the treatment choice and the outcome of an experiment. The context is represented as a $d$-dimensional real-valued vector throughout this paper.
\item A \emph{treatment} $T_i \in \cT = \{0, 1\}$, which is an intervention performed in an observed context. Treatments are assumed to be binary variables throughout this paper,\footnote{\edit{We discuss generalizations of our results to non-binary settings in Section~\ref{sec:discussion}.}} where $T_i=1$ is usually referred to as ``treatment'' and $T_i=0$ as ``control.'' 
\item An \emph{outcome} $Y_i\in \cY \subset \bbR$, measuring an aspect of interest of a patient, such as blood pressure or mortality, after the administration of a treatment, represented by a real-valued variable.
\item \emph{\edit{Partially observed} potential outcomes} $Y_i(0), Y_i(1)\in \cY \subseteq \mathbb R$ that correspond to the outcomes that \emph{would} be observed for unit $i$ under treatments $T=0$ and $T=1$, respectively. We assume that $Y_i=Y_i(T_i)$ throughout the paper, capturing both \emph{consistency} and \emph{non-interference}, also known in conjunction as the stable unit treatment value assumption (SUTVA), which is key to the existence of potential outcomes and the relevance of this hypothetical construct to the actual data \citep{rubin2005causal}.
\end{itemize}

{\blockedit%
Access to both potential outcomes $Y_i(0)$ and $Y_i(1)$ would allow us to retroactively determine which treatment $t\in \{0,1\}$ \emph{would have been} preferable for unit $i$, or proactively predict which one \emph{will be}. If the individual treatment effect (ITE), $Y_i(1)-Y_i(0)$ is positive and higher outcomes are preferred, $t=1$ is the better choice for $i$. However, the potential outcome of the treatment not administered, $Y_i(1-T_i)$, is an \emph{unobserved counterfactual}, which is the key impediment to assessing the ITE. That we do not observe what would have happened if we did something differently is often termed \emph{the fundamental problem of causal inference}. Without strong assumptions, the best we can hope for is instead to infer or predict the following quantities~\citep{kunzel2017meta}. 
}
\begin{mdframed}[innerbottommargin=.8em,innertopmargin=0em]%
\begin{thmdef}[Main estimands]
The expected potential outcomes $\mu_t$, for treatments $t \in \{0,1\}$, conditioned on a context $X=x$, are
\begin{equation}\label{eq:pot_def}
\mu_t(x):=\E_{Y(t)\mid X}[Y(t)\mid X=x],\quad \mbox{ for } t \in \{0,1\}~,
\end{equation}
and the conditional average treatment effect (CATE) given a context $x$ is
\begin{equation}\label{eq:cate_def}
\tau(x) := \E_{Y(0),Y(1) \mid X}\left[ Y(1) - Y(0) \mid X = x \right]=\mu_1(x)-\mu_0(x)~.
\end{equation}
\end{thmdef}
\end{mdframed}
The CATE is an object of key interest as it tells us what is the best prediction of the effect on an individual given only their context variables. This has a variety of uses. One important use is the personalization of treatment to make sure that the treatment is effective for the target. We illustrate $\mu_0, \mu_1$ and $\tau$ in Figure~\ref{fig:example} in Section~\ref{sec:theory}.

{\blockedit%
We consider learning from a simple random sample of size $m$ from a population distributed according to a population distribution $p(X,T,Y)$. The observations of a unit $i$ are denoted $(x_i, t_i, y_i)$. These data are assumed to be sampled i.i.d. from $p$. For this reason, we drop the subscript $i$ when dealing with the distribution of any single such random draw from this population. Note that since counterfactual outcomes are unobserved, our data consists just of $(X_1,T_1,Y_1),\dots, (X_m,T_m,Y_m)$, and $Y(1-T)$ remains unobserved.
Following a long tradition, observed samples from $p(X \mid T=0)$ and $p(X \mid T=1)$ are referred to as the control and treated groups, respectively.  With slight abuse of terminology, we use these labels also in reference to the conditional densities themselves. For convenience, for $t\in \{0,1\}$, we introduce the short-hands $p_t(X) := p(X \mid T=t)$ and $\hp_t(X)$ analogously. We assume that $p(X), p_1(X), p_0(X)$ are all probability measures on $\cX$, absolutely continuous w.r.t. the Lebesgue measure. 
}

\begin{figure}
    \centering
    \includegraphics[width=.95\textwidth]{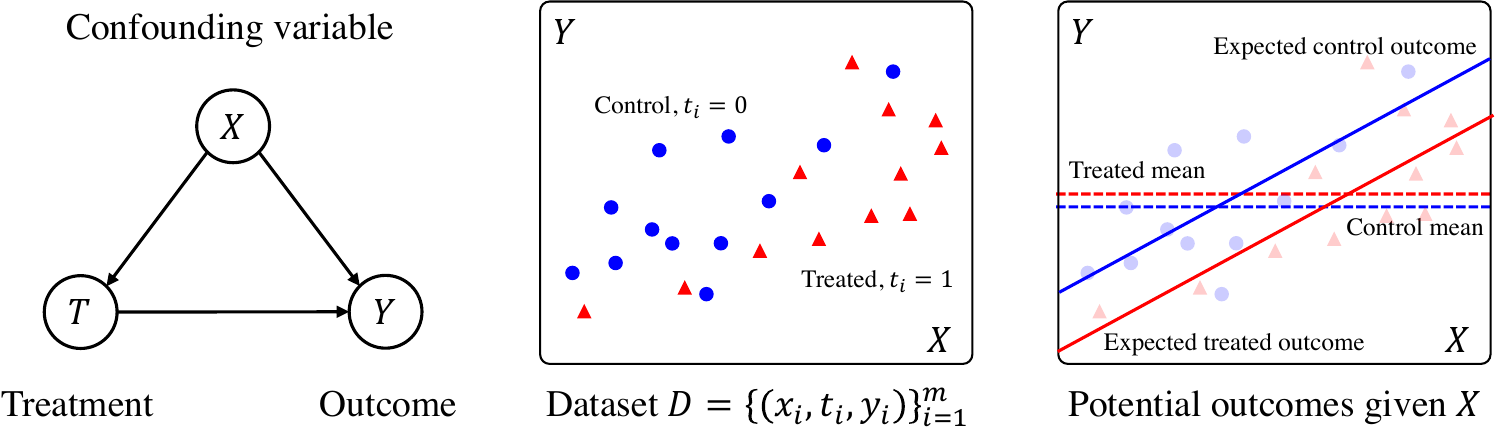}
    \caption{{\blockedit Effects of confounding in observational data. A variable $X$ influences both treatment assignments and outcomes (left) observed in the dataset $D$ (middle). This manifests as a distribution difference in $X$ between control (blue) and tread (red) observations. Despite the causal effect being negative for all $X$ (illustrated by the red solid line being below the blue solid line, right), the average outcome in the control group (dashed, blue) is lower than the average outcome in the treated group (dashed, red).}}
    \label{fig:confounding}
\end{figure}

{\blockedit%
\paragraph{Treatment group imbalance and confounding bias} In practical applications, different treatments are given in different contexts. For example, more expensive drugs may only be available to patients with sufficient insurance coverage. As a result, a central difficulty in working with observational data, unlike in randomized experiments, is that treatment and control groups can not be assumed equivalent. When treatments $T$ are selected based on the observed context $X$, this presents as a distributional difference, $p_0(X) \neq p_1(X)$. We illustrate treatment group imbalance and the effect it may have in Figure~\ref{fig:confounding}. If not adjusted for, this imbalance may result in confounding bias and inflated variance in estimates of potential outcomes and causal effects. Mitigating this bias and inflated variance in risk minimization is one of the main challenges addressed in this work.} 

\paragraph{Aside: Individual treatment effect} The conditional average treatment effect conditioned on everything that is known about a subject captures \edit{a causal effect which is as specific to an individual as the data supports}, rather than a population-level causal effect. It is therefore sometimes called \emph{individual treatment effect} (ITE)~\citep{johansson2016learning,shalit2016estimating}. While this definition aligns with concepts used in machine learning, it overloads existing terminology. In causal inference, the ITE has been previously defined as the difference $Y_i(1)-Y_i(0)$. The distinction between this ITE and CATE is that ITE is unique to an individual and may not be described exactly by any set of features $X$. For this reason, we adopt the more precise label CATE for the feature-conditional treatment effect function $\tau(x)$.

%
%
\subsection{Identifying assumptions}
\label{sec:assumptions}
Expected potential outcomes $\mu_t$ and conditional effects $\tau$ are said to be \emph{identifiable} if they can be uniquely computed from the distribution $p(X,T,Y)$ of the observed data. This is important because that distribution is the most one could ever hope to learn from an i.i.d. sample from $p$, and so anything that cannot be learned from the distribution cannot be learned from i.i.d. samples from it.
Without additional assumptions \edit{on $p$, the functions $\mu_t$ and $\tau$} may not be identifiable. To see this, consider an observational study where  treatment was given only to subjects over the age of 30, and the control group consists only of subjects under the age of 30. \edit{If age has an effect on a potential outcome of interest, there is no guarantee that it can be estimated for all subjects from such data.}

Sufficient conditions for identification have been studied in both very general settings~\citep{pearl2009causality,rubin2005causal} and in special cases that are commonly accepted in real-world applications. In \edit{most} of this work, we adopt and refer to the following assumptions.

\begin{thmasmp}[Ignorability]\label{asmp:ignorability}%
  The potential outcomes $Y(0), Y(1)$ and the treatment $T$ are conditionally independent given $X$,
  $$
  Y(0), Y(1) \indep T \mid X~.
  $$%
\end{thmasmp}%
\edit{Ignorability is often imprecisely (see remark below) called the \emph{no unmeasured confounders} assumption as it holds if all variables which affect both treatment and potential outcomes are included in $X$ \emph{and} if these satisfy additional conditions.}

\begin{thmasmp}[Overlap]\label{asmp:overlap}%
  In any context $x \in \cX$,
  any treatment $t \in \{0,1\}$ has a non-zero probability of being observed in the data
  $$
  \forall x\in \cX, t\in \{0,1\} : p(T=t \mid X=x) > 0~.
  $$
  \edit{Equivalently, the population $p(X)$ is absolutely continuous w.r.t. $p_0(X), p_1(X)$.}
\end{thmasmp}
Overlap is sufficient to ensure that knowledge of the outcomes in one treatment group may be generalized to the opposite group given access to a large enough sample size. Note that overlap only requires that the supports of the treatment groups are equal, not that they have similar densities. The degree to which treatment group densities are equal on this support is sometimes referred to as \emph{balance}.

Under Assumptions~\ref{asmp:ignorability}, \ref{asmp:overlap} and SUTVA ($Y=Y(T)$), expected potential outcomes and  the conditional average treatment effect are identifiable as can be seen by the simple identity
\begin{align*}
\tau(x) & = \E_{Y(0),Y(1) \mid X}\left[ Y(1) - Y(0) \mid X = x \right] \\
& = \E_{Y \mid X,T}\left[ Y \mid X = x, T = 1 \right] - \E_{Y \mid X,T}\left[ Y \mid X = x, T = 0 \right]~.
\end{align*}

\paragraph{Plausibility of the \edit{identifying} assumptions}
Both ignorability and SUTVA are assumptions that are fundamentally untestable given observational data alone~\citep{pearl2009causality}. Despite this, they are often made in practice to justify subsequent analysis, or make clear its potential limitations. A common heuristic motivation for ignorability is related to the richness of the variable $X$. The richer the data, the more likely are they to cover all confounding variables. It should be noted, that even if all confounders are measured, adjusting for some of them may introduce additional estimation bias or variance nonetheless~\citep{ding2017instrumental}. Furthermore, the overlap assumption becomes increasingly difficult to both satisfy and check as the dimensionality of $X$ grows~\citep{d2017overlap}.

%
%

\section{Related work}
\label{sec:related}
\edit{Observational causal inference has received considerable renewed attention with the recent growth of research into machine learning and artificial intelligence. We survey recent results and briefly cover relevant fundamental work below.}
%
%
The field may be broadly divided into two categories, dedicated to causal discovery and causal effect estimation, respectively. In the former, the direction and presence of causal relationships between observed variables is unknown, and the task is to infer them from data
\citep{geiger2015causal,spirtes1991algorithm,hoyer2009nonlinear,eberhardt2008causal,hyttinen2014constraint,silva2006learning}. %
In the latter, which is the setting of this work, the structure of causal relationships is assumed to be known: confounders $X$ are causal of treatment $T$ and outcome $Y$; treatment $T$ is causal of $Y$ (unless the effect is 0); any unmeasured variable is causal only of either $X$, $T$ or $Y$. We are primarily interested in estimating the \emph{conditional} average treatment effect of the treatment $T$ on the outcome $Y$ conditioned on the context $X$ \citep{johansson2016learning,shalit2016estimating,athey2016recursive,wager2018estimation,pearl2017detecting,abrevaya2015estimating,bertsimas2017personalized,green2010modeling,alaa2018limits}.

%
%
\paragraph{Outcome regression estimators}
Under ignorability, CATE is identified by the difference of two \edit{conditional expectations: the expected outcome under treatment conditioned on covariates minus the expected outcome under control conditioned on covariates. In an early influential paper, \citet{hill2011bayesian} demonstrated the power in separately estimating each conditional expectation, utilizing Bayesian Additive Regression Trees~\citep{chipman2010bart} for this task. The approach of fitting treatment-specific outcome regressions was later referred to as a} ``T-learner'' \citep{kunzel2017meta}, where T stands for \emph{two}, see Figure~\ref{fig:tlearner}. Another approach, classically used with linear models, is to fit a single regression model from the covariates $X$ \emph{and} the treatment $T$ to the outcome $Y$. CATE is then estimated by evaluating the difference between the prediction for $X,T=1$ and $X,T=0$. This approach is known as ``S-learner'', where S stands for \emph{single}. \edit{It has been argued that both methods are prone to \emph{compounding bias} when applied in high-dimensional, strongly confounded, small-sample settings which may require substantial regularization~\citep{nie2021quasi, hahn2020bayesian}. }
%
%
%
Rather than regressing on potential outcomes jointly or separately, a variety of work has studied directly estimating their difference, e.g. using trees \citep{athey2016recursive} and forests \citep{wager2018estimation}. Other work has studied meta-learners that combine base learners of the underlying regression functions using methods which are different from simple differencing in efforts to be more sample efficient~\citep{robins2000marginal,kunzel2017meta}. \edit{Double/debiased machine learning~\citep{chernozhukov2017double}, pseudo-outcome regression~\citep{kennedy2017non,kennedy2020optimal} and the R-learner~\citep{nie2021quasi} also attempt to mitigate bias by separately fitting non-parametric nuisance functions and effect estimates based on these. Multiple works have shown that when having a well-specified or non-parametric estimator for each regression, CATE estimation can be asymptotically consistent and/or asymptotically normal with optimal convergence rates, under appropriate assumptions \citep{chernozhukov2017double,belloni2014inference,nie2021quasi}.}

%
%
{\blockedit
\paragraph{Representation learning}
Representation learning has become a cornerstone of machine learning~\citep{bengio2013representation}, and is a natural component of causal effect estimators. In \citet{johansson2016learning,shalit2016estimating}, we used representation learning to share information about the outcome across treated and controls, while regularizing representation distributional distance between the groups. \citet{zhang2020learning,johansson2019support} recognized that such regularization may result in loss of ignorability in the representation and proposed to learn representations in which context information is preserved but where treatment groups overlap. \citet{hassanpour2019learning} identified the potential value in learning representations which disentangle parts of the contexts which are causes of a) treatment assignment, b) potential outcomes or c) both. By fitting a generative adversarial network, \citet{yoon2018ganite} estimate the uncertainty in predicted counterfactuals using their GANITE model---typical regression estimators, including representation learning methods, model only the conditional expectation of the outcomes given the input. \citet{chen2019deep} used multi-task representation learning to estimate treatment effects from electronic healthcare records. In this work, we show that simultaneous representation learning and re-weighting may overcome the issue of destructive representations encountered in early deep learning approaches to CATE estimation. A related idea was used by \citet{kallus2020deepmatch} to balance treatment groups using weights defined in a representation space learned using a deep network adversary.
}

%
%
{\blockedit 
\paragraph{Sample weighting estimators}
A complement to regression estimation of CATE is \emph{sample weighting}, where the goal is to alleviate systematic differences in baseline covariates across treatment groups. This idea is used in propensity weighting methods~\citep{austin2011introduction,rosenbaum1983central,jung2020estimating} which use the observed treatment policy to re-weight samples for causal effect estimation, and more generally in re-weighted regression, see e.g.~\citep{freedman2008weighting,jung2020learning}. Two challenges with propensity weighting methods are 1) the need to estimate the propensity score when it is unknown, and 2) the high variance introduced when the propensities are small, so that minute estimation errors lead to dividing by near-zeros. Interestingly, solving one of the above problems may mitigate also the other; \citet{hirano2003efficient,abadie2016matching} showed that weighting or matching using the estimated propensity score may be more favorable than using the true propensity score when the increase in bias is smaller than the reduction in variance. We observe a similar phenomenon in our proposed regularization scheme. \citet{gretton2009covariate,kallus2016generalized,kallus2017framework} proposed learning sample weights by minimizing a distributional distance between weighted samples from different treatment groups. These rely on specifying the data representation a priori, without regard for which aspects of the data matter for outcome prediction. We build on the sampling weighting (importance sampling) literature by developing theory and algorithms for weighted risk minimization for potential outcomes and CATE, both in a fixed representation and one learned from data. A similar idea was also used in concurrent work by~\citet{kallus2020deepmatch}. 
}

%
%
\paragraph{Learned prognostic scores}
Our work on representation learning has conceptual ties to the idea of the \emph{prognostic score} \citep{hansen2008prognostic}. A prognostic score is any function $\Phi(X)$ of the context $X$ that Markov separates $Y(0)$ and $X$, such that $Y(0) \indep X \mid \Phi(X)$. An extreme example is $\Phi(X) = X$. If $Y(0)$ follows a generalized linear model, then $\Phi(X) = \E\left[Y(0)|X\right]$ is also a prognostic score.
The prognostic score is a form of  dimension reduction which under certain assumptions is sufficient for causal inference. Note that unlike the propensity score, the prognostic score might very well be vector valued. One can view our approach as attempting to find approximate non-linear prognostic functions for both $Y(0)$ and $Y(1)$. We stress the approximate, because in fact we trade off how well our learned representation $\Phi(\cdot)$ is sufficient to explain the potential outcomes with a balancing objective which we show is important for good \emph{finite-sample} estimation of CATE.

%
%
\paragraph{Assumptions on observational data}
Estimation of causal effects from observational data is mostly performed under the assumptions of \emph{ignorability}, treatment group \emph{overlap} and \emph{consistency}, as they are otherwise generally \emph{unidentifiable}. In this work, we are motivated by the ignorability assumption throughout, but give several results that hold in its absence. For work on CATE estimation without ignorability, see e.g.,
\citet{kallus2018confounding,kallus2018removing,kallus2018interval,louizos2017causal,rosenbaum2002overt,Jesson21quant}. %
In contrast, we focus on the case where overlap is only partially satisfied. Lack of overlap is widely acknowledged as a problem~\citep{d2017overlap} but estimation in this setting has received considerably less attention in the literature.

{\blockedit
\paragraph{Learning guarantees}
In this work, we provide generalization bounds for the estimation of potential outcomes as functions of an observed context, based on the distributional difference between treated and controls.  \citet{alaa2018limits} showed that the best achievable minimax rate for non-parametric estimation in this task is independent of treatment group differences (selection bias). However, this main result is limited to the asymptotic regime. In low-sample settings, the error in estimating CATE is still affected by selection bias, and the best estimator may be one that adjusts for it. \edit{\citet{kunzel2017meta,nie2021quasi,kennedy2020optimal} also studied the minimax rate of CATE estimators and proposed methods which combine estimates of nuisance functions which achieve the optimal rate under smoothness conditions.} In contrast, estimates based on parametric or otherwise restricted hypothesis classes may not approach the optimal minimax rate (see paragraph below). Non-parametric estimation was also studied by~\citet{curth2021nonparametric}, in the context of different meta-learners (e.g., T-learner, S-learner). They find, similarly, that the best asymptotic estimators and the best low-sample estimators are not the same in general. \edit{\citet{athey2016recursive} and \citet{wager2018estimation} gave early results on valid confidence intervals for non-parametric estimation of heterogeneous causal effects.}
}

%
%
\paragraph{Best-in-class estimation}
{\blockedit
In this work, we are primarily interested in guarantees which hold also for misspecified models, functions that may not be able to exactly fit the outcome in terms of high-dimensional baseline covariates and treatment. \emph{Agnostic machine learning} focuses on finding best-in-class models and bounding the generalization error of any model, whether well-specified or not \citep{vapnik2013nature,cortes2010learning}. However, in the causal inference setting, under model misspecification, regression methods may suffer from additional bias when generalizing across populations subject to different treatments. Our work addresses this issue by extending specification-agnostic generalization bounds to the CATE estimation problem. These bounds motivate our algorithms in the same way that standard supervised learning generalization bounds motivate structural risk minimization \citep{vapnikbook}. As we will discuss later, the setting most closest to ours is unsupervised domain adaptation (UDA)~\citep{blitzer2008learning}, in which models must generalize from a source domain to a target domain, having observed labeled samples only from the former. Agnostic machine learning bounds for UDA have a rich literature, spanning results for regression~\citep{cortes2011domain} and classification~\citep{blitzer2008learning}, using the PAC~\citep{cortes2010learning,ben2012hardness,mansour2009domain} and PAC-Bayes~\citep{germain2020pac} frameworks, and in combination with representation learning~\citep{ben2007analysis,long2015learning}. We cannot provide a complete survey here but return briefly to discuss the relation between our work and domain adaptation in Section~\ref{sec:domainadaptation}.
}

%
%

%
%
\section{\edit{Generalization bounds for estimating potential outcomes and CATE}}
\label{sec:theory}
{\blockedit Our goal is to accurately predict potential outcomes $Y(t)$, and the CATE $\tau$} without making parametric assumptions on the functional form of the true potential outcomes. For this reason, we adopt the risk minimization approach to learning and search for best-in-class hypotheses. \edit{This requires generalization from the treated and control groups to the general population. Next, we derive bounds on the risk of such hypotheses in the following steps:}
\begin{enumerate}
  \item We define prediction \emph{risk} with respect to potential outcomes $\mu_0, \mu_1$ and relate it to the expected error in estimates of the CATE, $\tau$
  \item We show how the risk on the observed distribution is a \emph{biased} estimate of the desired \edit{population} risk and give sample re-weighting schemes that removes this bias.
  \item We give bounds on the expected risk under imperfect re-weighting schemes by placing assumptions on the loss with respect to the true outcome.
  \item We derive finite-sample versions of these bounds and combine them to form a single fully observable bound on the risk in estimates of potential outcomes and CATE.
\end{enumerate}
Our main generalization bound does not depend on treatment group overlap (Assumption~\ref{asmp:overlap}). This diverges from most theoretical results for treatment effect estimation and provides intuition for when we can expect extrapolation to succeed approximately. Consistent non-parametric estimation, however, still requires overlap. In Section~\ref{sec:rep_bound}, we extend these results to \edit{estimators which make use of representation learning}.

\subsection{Risk for hypotheses of potential outcomes and CATE}
We study prediction of potential outcomes $Y(t) \in \cY$, for binary treatments $t \in \{0,1\}$. \edit{We assume that outcomes are bounded, $\cY \subset [-\omega, \omega]$ for finite $\omega > 0$}. Predictions are made using hypotheses $f_t \in \cH$, from some hypothesis class \edit{$\cH \subset \{ h : \cX \rightarrow \cY\}$}. These can then be combined to form estimates of CATE,
$$\htau(x) \coloneqq f_1(x) - f_0(x)~.$$ We note that while this is not the only way to estimate $\tau$ (see e.g., \citep{robins2000marginal,kunzel2017meta,nie2021quasi} for alternatives), it does allow us to leverage separate bounds on the risk of hypotheses $f_0, f_1$ with respect to the potential outcomes, to then bound the risk of $\htau$.  We define the risk of hypotheses $f_0, f_1$ below.
%
%
\begin{mdframed}[innerbottommargin=.8em,innertopmargin=0em]%
\begin{thmdef}
Let $L : \cY \times \cY \rightarrow \mathbb{R}_+$ be a loss function, such as the squared loss $L(y, y') = (y - y')^2$.
  The expected \emph{pointwise loss} of a hypothesis $f_t$ at a point $x$ is:
  \begin{equation}
    \ell_{f_t}(x) \coloneqq \E_{Y(t)\mid X}[L(Y(t), f_t(x)) \mid X=x]~.
  \end{equation}
  The \emph{marginal risk} of a hypothesis $f_t$ w.r.t. a population $p(X)$ is
  \begin{equation}\label{eq:marg_risk}
    R(f_t) \coloneqq \E_X\left[\ell_{f_t}(X) \right]~,
  \end{equation}
  and the \emph{factual risk} of $f_t$ w.r.t. treatment group $p(X\mid T=t)$ is
  \begin{equation}
    R_t(f_t) \coloneqq \E_{X\mid T}\left[\ell_{f_t}(X) \mid T=t \right]~.
  \end{equation}
  The \emph{counterfactual risk} is $R_{1-t}(f_t) := \E_{X\mid T}\left[\ell_{f_t}(X) \mid T=1-t \right]$. The subscript on the risk $R_t$ indicates the treatment group over which it is evaluated. Note that the potential outcome against which the risk is evaluated is implicit in this notation---we only consider evaluating $f_t$ against $Y(t)$ or $\mu_t$, $\hat{\tau}$ against $\tau$, et cetera.
\end{thmdef}
\end{mdframed}
In most of this work, we restrict our attention to the squared error loss but note that our analysis generalizes to other convex loss functions, such as the mean absolute deviation.

Similar to potential outcomes, we assess the quality of an estimate $\htau$ of $\tau$ based on the expectation of the loss function $L$ over the marginal density of covariates, $p(X)$,
\begin{equation}
  R(\htau) \coloneqq \E_{X}\left[ L(\tau(X), \htau(X)) \right]~.
  \label{eq:taumse}
\end{equation}

The marginal risk $R(\htau)$ is the overall expected error in estimating CATE, taken over the entire population. However, $R(\htau)$ is not readily computable from data because neither $\tau(X)$ nor $p(X)$ are known. Moreover, we cannot make an empirical average estimate of it because, again, neither $\tau(X_i)$ nor $Y_i(1)-Y_i(0)$ are known. Instead, we will bound $R(\htau)$ from above. The main challenge of computing the marginal risk for hypotheses of potential outcomes is to quantify the counterfactual risk, and this is the primary concern of this work.

Unlike $R(f_t)$, $R(\htau)$ does not depend on the noise (conditional variance) in $Y(t)$ \footnote{This is because that $\tau$ is defined in terms of expectations over the potential outcomes.}. We adopt this convention for $R(\htau)$ as it coincides with the Precision in Estimation of Heterogeneous Effects (PEHE)~\citep{hill2011bayesian}. However, similarly to $L(Y(t), f_t)$, $L(\tau, \htau)$, is not observed over $p$, as $Y(1)$ is only observed for the treated group $p_1$, and $Y(0)$ only for the control group $p_0$. We return to this issue later, and begin instead by stating the following result relating the risk of $\htau$ to those of $f_0$ and $f_1$, in the case of $L$ the squared loss.
%
%
\begin{mdframed}[innerbottommargin=.8em,innertopmargin=0em]%
\begin{thmlem}\label{lem:cate_pot}
Let $L(y, y') = (y-y')^2$ be the squared loss function. For hypotheses $f_0, f_1$ of expected potential outcomes $\mu_0, \mu_1$, with marginal risks $R(f_0), R(f_1)$, and $\htau = f_1 - f_0$, there is a constant $\sigma_Y$ (defined in the proof below), such that
  \begin{equation}%
  R(\htau) \leq 2\left( R(f_0) + R(f_1) \right) - 4\sigma^2_Y~.
  \end{equation}
Similar results hold \edit{also for other losses that satisfy the (relaxed) triangle inequality}, e.g., the absolute loss, $L(y,y') = |y - y'|$, \edit{with corresponding noise decompositions.}
\end{thmlem}
\end{mdframed}
\begin{proof}
Due to the relaxed triangle inequality for squared differences,
\begin{align*}
  \E_X[(\tau(X) - \htau(X))^2] & =
  \E_X[(\mu_1(X) - \mu_0(X) - f_1(X) + f_0(X))^2] \\
  & \leq 2\left(\E_X[(\mu_1(X) - f_1(X))^2] + \E_X[(\mu_0(X) - f_0(X))^2] \right)
\end{align*}
Now, by the standard bias-noise decomposition,
\begin{eqnarray*}
R(f_t) & = &  \E_{X,Y(t)}\left[ ((Y(t) - \mu_t(X))^2 \right] + \E_{X}\left[(\mu_t(X)- f_t(X)))^2 \right]~.
\end{eqnarray*}

Hence, for $t\in \{0,1\}$ with $\sigma^2_{Y(t)} \coloneqq \E_{X,Y(t)}\left[ ((Y(t) - \mu_t(X))^2 \right]$,
$$
\E_X[(\tau(X) - \htau(X))^2] \leq 2\left(R(f_0) - \sigma^2_{Y(0)} + R(f_1) - \sigma^2_{Y(1)} \right)
$$
and with $\sigma^2_Y \coloneqq \max(\sigma^2_{Y(0)}, \sigma^2_{Y(1)})$, we have our result.
\end{proof}

Lemma~\ref{lem:cate_pot} implies that small errors in hypotheses of potential outcomes guarantee small errors in CATE. However, it is worth noting that this decomposition need not lead to the best achievable bound in all cases. Even when $Y(0)$ and $Y(1)$ are complex functions, $\tau(x)$ may be a simple function. In this work, we do not address this in our theoretical treatment but find that sharing parameters in estimation of $Y(0)$ and $Y(1)$ lead to better results empirically. Next, we study $R(f_0)$ and $R(f_1)$ separately, in terms of observable quantities, to later give a self-contained result for $R(\htau)$.

%
%
\subsection{Importance-weighting hypotheses \& propensity scores}\label{subsec:imp_weigh}
We proceed to show how the marginal risk $R$ in potential outcomes and CATE may be computed by re-weighting the factual risk $R_t$. This approach is widely used within machine learning~\citep{shimodaira2000improving,cortes2010learning,rosenbaum1983central}. We note in passing that $R_t$ is not observed directly but can be readily estimated from an empirical sample. We return to this issue in later sections.

Due to the fundamental impossibility of observing counterfactual outcomes, each potential outcome $Y(t)$ is only observed for subjects who were given treatment $T=t$, distributed according to $p(X \mid T=t)$. As a result, unless treatment is assigned randomly (independently of $X$), $R_t(f_t)$ is different from $R(f_t)$ in general. In particular, a minimizer $f^*_t$ of $R_t(f_t)$ can be arbitrarily different from a minimizer of $R(f_t)$, depending on the difference between $p$ and $p_t$. This bias can have large impact on treatment policies derived from $f_0, f_1$ and $\tau$.
We illustrate this problem in Figure~\ref{fig:example}.

\begin{figure}[t!]
  \centering
  \begin{subfigure}{.333\textwidth}
    \centering
    \includegraphics[width=.96\textwidth]{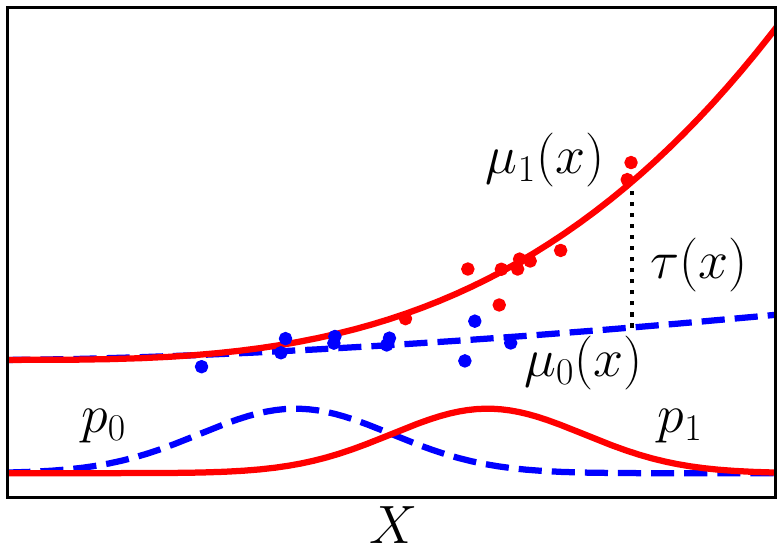}
    \caption{Outcomes and CATE}
  \end{subfigure}%
  \begin{subfigure}{.333\textwidth}
    \centering
    \includegraphics[width=.96\textwidth]{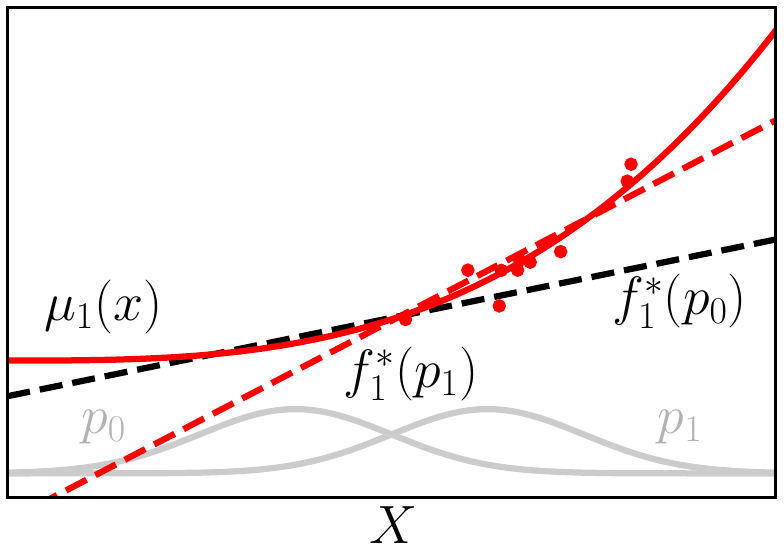}
    \caption{Misspecified models}
  \end{subfigure}%
  \begin{subfigure}{.333\textwidth}
    \centering
    \includegraphics[width=.96\textwidth]{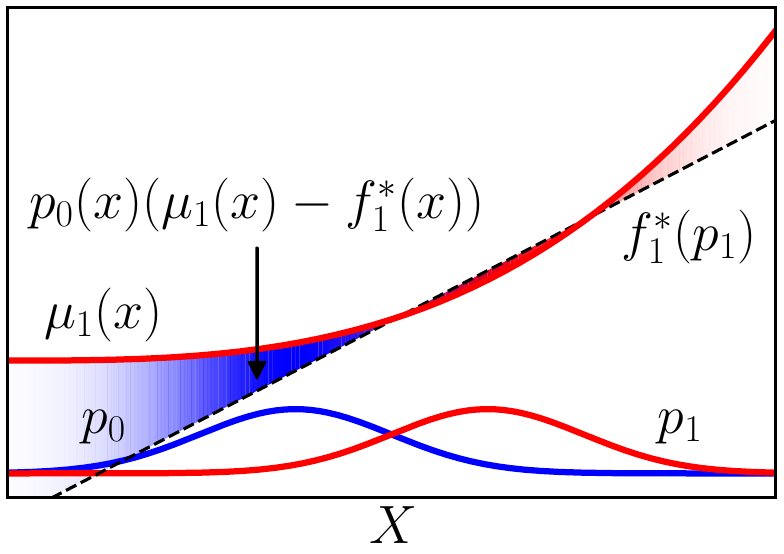}
    \caption{Error and risk}
  \end{subfigure}
  \caption{Illustrative example of bias in regression adjustment of expected potential outcomes $\mu_0(x), \mu_1(x)$ and CATE $\tau(x)$. In (a), we show the two potential outcomes and the two treatment groups $p_0, p_1$ in dashed blue and solid red lines, respectively, as well as samples of each group. In (b), we show the best \emph{linear models} $f_1^*(p_0), f_1^*(p_1)$ of the potential outcome under treatment $\mu_1(x)$ fit to the potential outcome of the control group and treated group respectively. In (c), we illustrate the difference in weighted error (bias) for the model fit to the treated group $f_1^*(p_1)$ evaluated in the control group and treated group.\label{fig:example}}
\end{figure}

To reduce the bias described above, a classical solution is to re-weight the observable risk~\citep{shimodaira2000improving} using a function $w : \cX \rightarrow \mathbb{R}_+$, such that $\E_{X\mid T}[w(X)\mid T=t] = 1$. We define the re-weighted population and treatment group risks as follows, %
\begin{align}\label{eq:risk_weighted}%
R^w(f_t) & := \E_{X}\left[w(X) \ell_{f_t}(X) \right] \\
R_t^w(f_t) & := \E_{X\mid T}\left[w(X) \ell_{f_t}(X)  \mid T=t\right]~, 
\end{align}
where $w$ is chosen to skew the sample to mimic the distribution of the full population $p$. Many common choices of weights are based on the family of \emph{balancing scores}~\citep{rosenbaum1983central}, of which the best known is the \emph{propensity score} $\eta(X)$ with respect to $X$,
\begin{equation}
  \eta(x) := p(T=1 \mid X=x)~.
\end{equation}
We can now state the following result.
%
%
\begin{mdframed}[innerbottommargin=.8em,innertopmargin=0em]%
\begin{thmlem}\label{lem:impsmp}%
For fixed $t\in \{0,1\}$, under Assumption~\ref{asmp:overlap} (overlap), there exists a weighting function $w : \cX \rightarrow \mathbb{R}_+$, such that  $R_t^w(f_t) = R(f_t)$. In particular, this holds for
\begin{equation}\label{eq:balancing}%
w(x) := \frac{p(T=t)}{(2t-1)\edit{\eta(x)}+1-t}~.
\end{equation}
More generally, it holds for \eqref{eq:balancing} with $\eta(\phi(x))$ for any $\phi$ such that $\ell_f \indep X \mid \phi(X)$. We refer to weights that satisfy \eqref{eq:balancing} as \emph{balancing weights}.
\end{thmlem}
\end{mdframed}
\begin{proof} For any weighting function $w$,
$$
R_t^w(f_t) = \int_{x \in \cX} w(x) \ell_{f_t}(x) p_t(x) dx = \int_{x \in \cX} \frac{p_t(x)}{p(x)} w(x) \ell_{f_t}(x) p(x) dx = R^{\frac{p_t}{p}w}(f_t)~.
$$
With $w(x) = p(x)/p_t(x)$, the special case in \eqref{eq:balancing} follows from Bayes theorem and the definition of $\eta(x)$. The second step uses Assumption~\ref{asmp:overlap} to ensure that $p(x)/p_tx(x)$ is defined. The more general statement follows from integration over $\phi$ and a change of variables.
\end{proof}
\begin{thmrem}[Violation of Assumptions~\ref{asmp:ignorability} \& \ref{asmp:overlap}]
If overlap is only partially satisfied, Lemma~\ref{lem:impsmp} may still be applied to the expected risk over the subset of $\cX$ for which overlap holds. More generally, weights may be chosen to emphasize regions where treatment groups are more similar~\citep{li2018balancing}. If ignorability is violated, such as when unobserved confounders exist, a consistent estimator could in theory be obtained by letting the weights $w$ depend also on $Y$. However, such weights are not identified from observed data. Instead, a worst-case bound may be obtained by searching over a family of weighting functions in which these optimal weights are members. This is the topic of sensitivity analysis (see e.g., \citet{rosenbaum2002overt} for a comprehensive overview).
\end{thmrem}

In practice, $\eta(X)$ and balancing weights $w$ are typically unknown and have to be estimated from data. Moreover, even though weights based on $\eta$ are optimal in expectation, they can lead to poor finite-sample behavior~\citep{swaminathan2015counterfactual}. For these reasons, \emph{even if we had knowledge of $\eta(X)$}, we are often interested in weighting functions that do not satisfy Lemma~\ref{lem:impsmp}. Next, we give bounds on the difference between the re-weighted (factual) empirical risk, under  \emph{arbitrary} weightings, and the marginal risk.

%
%
\subsection{Bounds on the risk of a re-weighted estimator}
\label{sec:weighted_risk}
When overlap is not satisfied everywhere or the chosen weighting function $w$ is not perfectly balancing, the difference between the weighted factual risk $R_t^w(f_t)$ and the marginal risk $R(f_t)$ may be arbitrarily large, without further assumptions on the potential outcomes or the hypothesis class $\cH$. However, in many cases we have reason to make assumptions about the worst-case loss in generalization, as is typical in statistical learning theory. In this section, we give bounds on $R(f_t)$ under such assumptions.

Let $\cL \subset \{\cX \rightarrow \mathbb{R}_{+}\}$ be a space of pointwise loss functions with respect to the covariates $X$ endowed with a norm $\|\cdot \|_\cL$. In this work, we assume that the expected conditional loss $\ell_{f_t}$ for each potential outcome belongs to such a family, i.e., that $\ell_{f_t} \in \cL$. A simple example of such a family is the set of loss functions with bounded maximum value, $\cL_M = \{\ell : \rightarrow \mathbb{R}_{+}\ ;\ \sup_{x\in \cX}\ell(x) \leq M\}$. This assumption is satisfied without loss of generality as long as the outcome $Y$ is bounded. However, it is not very informative and will lead to loose bounds in general. Instead, we may make assumptions about the functional properties of $\ell_{f_t}$. Such assumptions include that $\ell_{f_t}$ is $C$-Lipschitz or belongs to a reproducing-kernel Hilbert space (RKHS). We illustrate the former with an example in Figures~\ref{fig:pred_example}--\ref{fig:lipschitz_example}. Note that these assumptions are not testable without further assumptions such as smoothness. 

\begin{figure}[t!]
  \centering
  \begin{subfigure}{.333\textwidth}
    \centering
    \includegraphics[width=.95\textwidth]{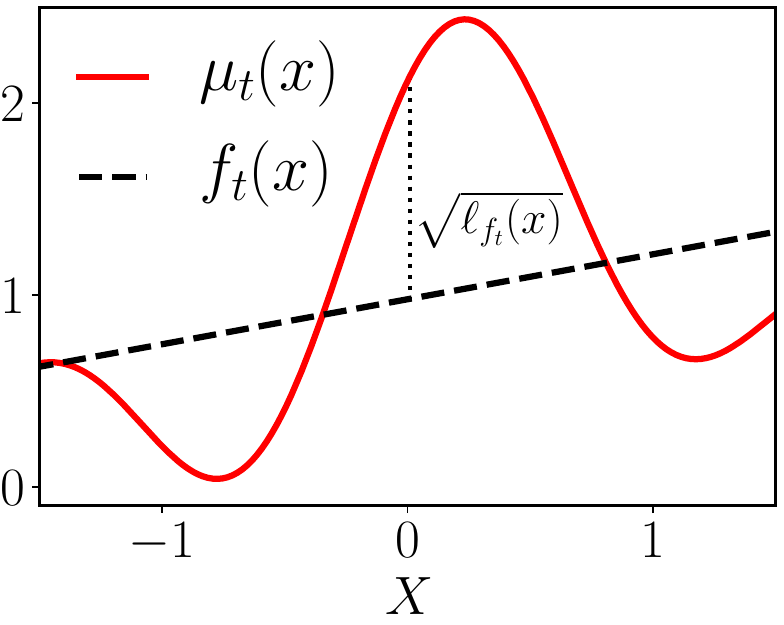}
    \caption{\label{fig:pred_example}Misspecified hypothesis $f_t$}
  \end{subfigure}%
  \begin{subfigure}{.333\textwidth}
    \centering
    \includegraphics[width=.95\textwidth]{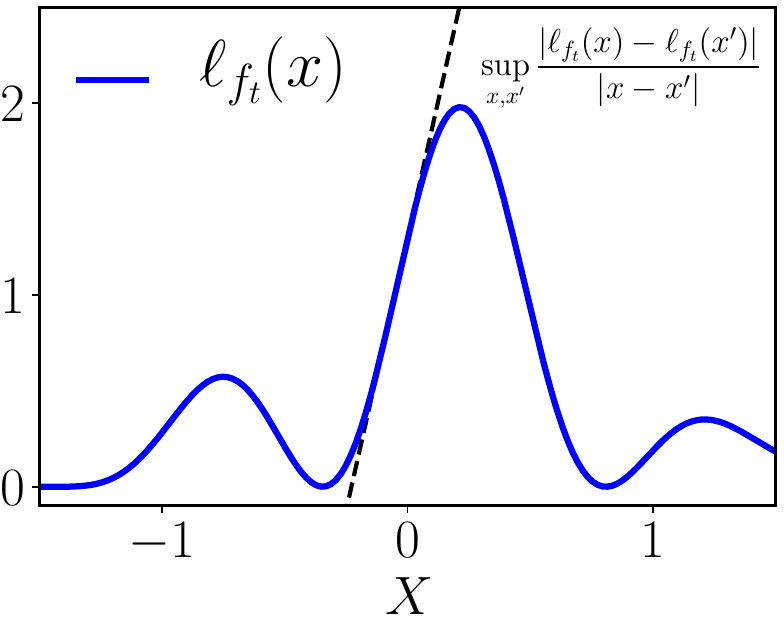}
    \caption{\label{fig:lipschitz_example}Lipschitz loss $\ell_{f_t}$}
  \end{subfigure}%
  \begin{subfigure}{.333\textwidth}
    \centering
    \includegraphics[width=.95\textwidth]{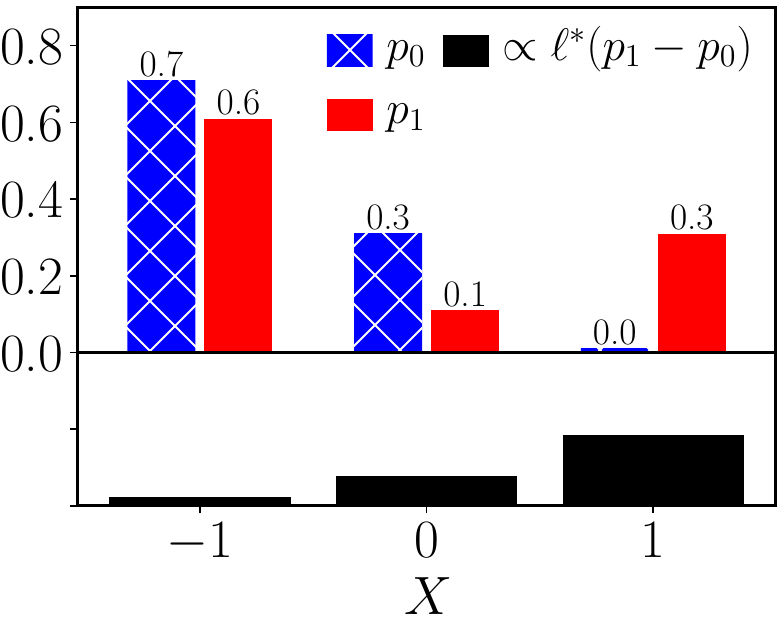}
    \caption{\label{fig:mmd_example2}Bound on loss difference}
  \end{subfigure}
  \caption{Example illustrating assumptions on the pointwise loss $\ell_{f_t}$ in a noiseless setting. In (a) we see the true potential outcome $\mu_t$ and a hypothesis $f_t$. The pointwise loss between them is plotted in (b). In (c), we illustrate the difference between two densities $p_0$ and $p_1$ on $\{-1, 0, 1\}$. The bottom panel shows the worst-case contribution of any loss function in an RBF-kernel RKHS $\cL$ to the \emph{difference} in risk $R_0(f_t) - R_1(f_t)$. The more similar $p_0, p_1$, or the smoother the functions in $\cL$, the smaller the overall contribution. \label{fig:mmd_example}}
\end{figure}

Now, consider the marginal distribution $p$ and a re-weighted treatment group $p^w_t$ on $\cX$, \edit{with $p^w_t(x) = p_t(x)w(x)$ for a weighting function $w$ with mean 1 over $p_t$ as in $\eqref{eq:risk_weighted}$.} Let $\ell \in \cL$ be a pointwise loss on $\cX$. Recall that $R(f_t)$ and $R_t^w(f_t)$ denote the marginal and re-weighted factual risks respectively.  By definition, 
\begin{eqnarray}%
R(\ell) \coloneqq \E_{X \sim p}[\ell] & = & \E_{X\sim p_t^w}[\ell] + \int_{x \in \cX} \ell(x)(p(x) - p^w_t(x))dx \nonumber \\ \label{eq:ipm_pre}
& \leq & \E_{X\sim p_t^w}[\ell] + \sup_{\ell' \in \cL }\left| \int_{x \in \cX} \ell'(x)(p(x) - p^w_t(x))dx \right|~.
\end{eqnarray}%

The second term on the right-hand side in \eqref{eq:ipm_pre} is known as the \emph{integral probability metric} distance (IPM) between $p$ and $p^w_t$ w.r.t. $\cL$, defined as follows \citep{muller1997integral} :
\begin{equation}
\ipm_{\cL}(p,q) := \sup_{\ell \in \cL} \left| \E_{p}[\ell(x)] - \E_{q}[\ell(x)] \right|~.
\end{equation}
Particular choices of $\cL$ make the IPM equivalent to different well-known distances on distributions: With $\cL$ the family of functions in the norm-1 ball in a reproducing kernel Hilbert space (RKHS), $\ipm_{\cL}$ is the Maximum Mean Discrepancy (MMD)~\citep{gretton2012kernel}; When $\cL$ is the family of functions with Lipschitz constant at most 1, we obtain the Wasserstein distance~\citep{villani2008optimal}. As pointed out by e.g., \citet{long2015learning}, the IPM may be viewed as the loss of a treatment group classifier, and adversarial losses may be considered in its place~\citep{ganin2016domain}. In Figure~\ref{fig:mmd_example2}, we illustrate the maximizer $\ell^*$ of the supremum, in terms of its contribution to the expected difference in risk in the MMD case.

Before stating the final form of our bounds on $ R(f_t) - \hat{R}_t^w(f_t)$, we note that for $t \in \{0,1\}$, with $\pi_t = p(T=t)$, we may decompose the population risk $R$ as follows.
\begin{eqnarray}
R(f_t) & = & 
\pi_t \underbrace{R_t(f_t)}_{\textnormal{Observable}} + (1-\pi_t)\underbrace{R_{1-t}(f_t)}_{\textnormal{Unobserved}} ~.
\label{eq:RpvsRt}
\end{eqnarray}
The factual risk $R_t(f_t)$ is identifiable under ignorability, as
$$
\ell_{f_t}(X) = \E_{Y(t)\mid X}[L(f_t(X), Y(t)) \mid X]
= \E_{Y\mid X, T}[L(f_t(X), Y) \mid X, T=t]
$$
For this reason, to bound the risk of $f_t$ on the whole population $p$ it is sufficient for us to bound the \emph{counterfactual risk} $R_{1-t}(f_t)$, and estimate $R_t(f_t)$ empirically.
%
%
\begin{mdframed}[innerbottommargin=.8em,innertopmargin=0em]%
\begin{thmlem}\label{lem:ipm_bound}
{\blockedit%
For a hypothesis $f$ with expected point-wise loss $\ell_{f}(x)$ such that $\ell_{f}/\|\ell_{f}\|_{\cL}  \in \cL$ and any re-weighting $w \in \{w' : w'(x) \geq 0, \E_X[w'(X) \mid T=t] = 1\}$, with $t\in \{0,1\}$,
\begin{align*}
\arraycolsep=1.4pt\def\arraystretch{1.4}
R_{1-t}(f) - & R_t^w(f) \;\leq\;  \|\ell_{f}\|_{\cL}\mipm_{\cL}(p_{1-t}^{\vphantom{w}}, p^w_t)~,
\end{align*}
where $p_t(x) = p(X=x \mid T=t)$. Further, there exist re-weightings $w^*$ such that,
\begin{align}
\arraycolsep=1.4pt\def\arraystretch{1.4}
R_{1-t}(f) - & R_t^{w^*}(f) \;\leq\;  \|\ell_{f}\|_{\cL}\mipm_{\cL}(p_{1-t}^{\vphantom{w^*}}, p^{w^*}_t)
 \;\leq\; \|\ell_{f}\|_{\cL}\mipm_{\cL}(p_{1-t}, p_t)~.
\label{eq:ipm_bound}
\end{align}
}
The first inequality is tight under Assumption~\ref{asmp:overlap} for weights $w(x) = p_1(x) / p_0(x)$. The second is not tight for general $f$ unless $p_0 = p_1$.%

\end{thmlem}
\end{mdframed}
\begin{proof}
  The result follows from the definition of $\ipm$s. \edit{We prove it for $t=0$.}
  \begin{align}
    R_1(f) - R_0^w(f) &=
    \E_{X\mid T}[\ell_{f}(X) \mid T=1] - \E_{X\mid T}[w(X) \ell_{f}(X) \mid T=0] \nonumber \\
    & \leq \left| \E_{X\mid T}[\ell_{f}(X) \mid T=1] - \E_{X\mid T}[w(X) \ell_{f}(X) \mid T=0] \right| \nonumber \\
    & \leq \|\ell_{f}\|_{\cL} \sup_{h\in \cL} \left| \E_{X\mid T}[h(X) \mid T=1] - \E_{X\mid T}[w(X) h(X) \mid T=0] \right| \label{eq:ipmstep}\\
    & = \|\ell_{f}\|_{\cL} \ipm_{\cL}(p_1, p_0^w) \nonumber~.
  \end{align}
  Step \eqref{eq:ipmstep} relies on that $\ell_f/\|\ell_f\|_{\cL}  \in \cL$.
  Further, for importance weights \edit{$w^*(x) = p_1(x)/p_0(x)$}, for any $\ell \in \cL$, under Assumption~\ref{asmp:overlap} (overlap),
  \begin{align*}
    & \E_{X\mid T}[\ell(X) \mid T=1] - \E_{X\mid T}[w^*(x) \ell(x) \mid T=0] \\
    & = \E_{X\mid T}[\ell(X) \mid T=1] - \E_{X\mid T}\left[\frac{p_1(x)}{p_0(x)} \ell(x) \mid T=0\right]  = 0
  \end{align*}
  and the first inequality in \eqref{eq:ipm_bound} is tight, as {\blockedit $\ipm_{\cL}(p_1^{\vphantom{w^*}}, p^{w^*}_0) = 0$}. Given that $\ipm \geq 0$ in general, the second inequality holds in this case as well. If overlap does not hold, the ratio $\frac{p_1(x)}{p_0(x)}$ is not defined on the support of $p_1(x)$. The result for $t=1$ follows analogously.
\end{proof}

%
%
\begin{thmcol}\label{col:ipm_bound}
Under the conditions of Lemma~\ref{lem:ipm_bound}, with $\tilde{w}(x) := \pi_t + (1-\pi_t)w(x)$,
\begin{equation}\label{eq:ipm_bound2}
R(f_t) \leq R_t^{\tilde{w}}(f_t) + (1-\pi_t)\|\ell_{f_t}\|_{\cL}\ipm_{\cL}(p_{1-t}^{\vphantom{w_t}}, p^w_t)
\end{equation}
\end{thmcol}
\begin{proof}
The result follows immediately from Lemma~\ref{lem:ipm_bound}.
\end{proof}

\begin{thmrem}[Necessity of assumptions]
Lemma~\ref{lem:ipm_bound} and Corollary~\ref{col:ipm_bound} do not strictly speaking depend on Assumption~\ref{asmp:ignorability} (ignorability) due to the definitions of $R(f_t)$ and $\ell_{f_t}$ being made w.r.t. the potential outcomes $Y(t)$. However, to estimate the right-hand side of \eqref{eq:ipm_bound2} from observational data, ignorability is required.  Moreover, neither result depend on Assumption~\ref{asmp:overlap} (overlap) as long as $w(x)$ is defined everywhere on $p_t(x)$. For particular losses, we can avoid making assumptions about $\ell_{f_t}$, by making assumptions on $f_t$ and the hypothesis class $\cH$ instead. This approach was taken by \citet{ben2007analysis}, who used the so-called ``triangle inequality for loss functions'' to give bounds on the risk in unsupervised domain adaptation under assumptions on $\cH$. However, this leads to the rather unattractive property that the resulting bounds are not tight even in the special case that $p_0 = p_1$.
\end{thmrem}

%
%
\subsection{Bounds based on finite samples}
Adopting results from statistical learning theory, we bound the difference between empirical estimates of $R_t^{w}(f_t)$ and $\ipm_{\cL}(p_{1-t}, p_t^{w})$ and their expected counterparts. These results are then combined to form a bound on $R(f_t)$.

The re-weighted risk $R_t^{w}$ may be estimated, for a fixed weighting function $w$ by the standard Monte-Carlo method. Consider a sample $\cD = \{(x_1, t_1, y_1), ..., (x_n, t_n, y_n)\} \sim p^n(X, T, Y)$, with $n_t = \sum_{i=1}^n \mathds{1}[t_i = t]$ and define the empirical weighted factual risk,
$$
\hR_t^{w}(f_t) := \frac{1}{n_t}\sum_{i : t_i = t} w(x_i) L(f_t(x_i), y_i)~.
$$
We aim to bound the difference
$$
\Delta(f_t) := R_t^{w}(f_t) - \hR_t^{w}(f_t)~.
$$
To achieve this, we use a result from the literature which builds on the concept of pseudo-dimension $\text{Pdim}(\cH)$ of a function class $\cH$, {\blockedit %
originally due to \citet{pollard1984convergence}. 
\begin{thmdef}[\citet{anthony2009neural}, Def. 11.1--11.2] %
Let $\cH$ be a set of functions mapping from a domain $\cX$ to $\bbR$ and suppose that $\cS = \{x_1, x_2, \ldots, x_m\} \subseteq \cX$. Then, $\cS$ is pseudo-shattered by $\cH$ if there are real numbers $y_1, y_2, \ldots, y_m$ such that for each $b \in \{0,1\}^m$, there is a function $h_b \in \cH$ with $\mbox{sgn}(h_b(x_i) - y_i) = b_i)$ for $1\leq i \leq m$. $\cH$ has pseudo-dimension $d$ if $d$ is the maximum cardinality of a subset $\cS$ of $\cX$ that is pseudo-shattered by $\cH$. If no such maximum exists, we say that $\cH$ has infinite pseudo-dimension. The pseudo-dimension of $\cH$ is denoted $\text{Pdim}(\cH)$. %
\end{thmdef}%
We have modified the result below to account for the variance in the outcome $Y$.
}
%
%
\begin{mdframed}[innerbottommargin=.8em,innertopmargin=0em]%
\begin{thmlem}[\citet{cortes2010learning}, \edit{Theorem~4}]\label{lem:cortes_bound}
Let $\ell_h = \E_{Y\mid X}[L(h(x), Y)\mid X=x]$ be the expectation of the squared loss $L(y,y') = (y-y)^2$ of a hypothesis $h \in \cH \subseteq \{h' : \cX \rightarrow \mathbb{R}\}$, let $d = \text{\emph{Pdim}}(\{\ell_h : h\in \cH\})$ and let $\sigma_Y^2 = \E_{X,Y}[L(Y, \E_{Y\mid X}[Y\mid X])]$. Then, for a weighting function $w(x)$ such that $\E_X[w(X)] = 1$, with probability at least $1-\delta$ over a sample $((x_1, y_1), ..., (x_n, y_n))$, with empirical distribution $\hat{p}$,
\begin{align*}
R^w(h) \leq \hat{R}^w(h) + V_{p,\hat{p}}[w(x)l_h(x)]\frac{\cC^{\cH}_n}{n^{3/8}} + \sigma_Y^2
\;\;\mbox{ where }\;\;
\cC^{\cH}_n = 2^{5/4} \left(d \log \frac{2ne}{d} + \log \frac{4}{\delta}\right)^{3/8}
\end{align*}
and $
V_{p,\hat{p}}[w(x)l_h(x)] = \max\left(\sqrt{\E_X[w^2(X)\ell_h^2(X)]}, \sqrt{\E_{X\sim \hat{p}}[w^2(X)\ell_h^2(X)]}\, \right)
$.
\end{thmlem}
\end{mdframed}
Lemma~\ref{lem:cortes_bound} applies to any valid weighting function $w$, not only importance weights or weights based on balancing scores. Used in conjunction with Corollary~\ref{col:ipm_bound}, Lemma~\ref{lem:cortes_bound} allows us to separate the bias (the IPM-term) and variance (see above) introduced by $w$.

The efficiency with which a sample may be used to estimate $\ipm_\cL$ depends on the chosen function family $\cL$. In particular, the sample complexity of learning the Wasserstein distance between two densities on $\cX$ scales as $\mathcal{O}(d)$ with the dimension $d$ of $\cX$, whereas the kernel-based MMD has $\mathcal{O}(1)$ dependence. Below, we state a result bounding the sample complexity for the MMD with universal kernels.
%
%
\begin{mdframed}[innerbottommargin=.8em,innertopmargin=0em]%
\begin{thmlem}[\citet{sriperumbudur2009integral}]
  Let $\cX$ be a measurable space. Suppose $k$ is a universal, measurable kernel such that $\sup_{x\in \cX} k(x,x) \leq C \leq \infty$ and $\cL$ the reproducing kernel Hilbert space induced by $k$, with $\nu := \sup_{x\in \cX,f\in \cL}f(x) < \infty$. Then, with $\hat{p}, \hat{q}$ the empirical distributions of $p, q$ from $m$ and $n$ samples, and with probability at least $1-\delta$,
  \begin{align*}
  \left|\mipm_{\cL}(p,q) - \mipm_{\cL}(\hat{p}, \hat{q}) \right|
  \leq \sqrt{18 \nu^2 \log \frac 4 \delta C} \left(\frac{1}{\sqrt m} + \frac{1}{\sqrt n} \right)~.
  \end{align*}
  \label{lem:mmd_approx}
\end{thmlem}
\end{mdframed}
The Gaussian RBF kernel $k(x, x') = e^{-\|x-x'\|^2_2 / (2\sigma^2)}$, with bandwidth $\sigma > 0$, is an important class of universal kernels to which Lemma~\ref{lem:mmd_approx} applies.

With Lemmas~\ref{lem:ipm_bound}--\ref{lem:mmd_approx} in place, we can now state our main result.
%
%
%
\begin{mdframed}[innerbottommargin=.8em,innertopmargin=0em]%
\begin{thmthm}
  Assume that ignorability (Assumption~\ref{asmp:ignorability}) holds w.r.t. $X$. Given is a sample $(x_1 , t_1, y_1)$, $...$, $(x_n, t_n, y_n) \overset{i.i.d.}{\sim} p(X, T, Y)$ with empirical measure $\hat{p}^n$ and $n_t := \sum_{i=1}^n\mathds{1}[t_i = t]$ for $t\in \{0,1\}$. Let $f_t(x) \in \cH$ be a hypothesis of $\E_{Y(t)\mid X}[Y(t)\mid X=x]$ and $\ell_{f_t}(x) := \E_{Y\mid X}[L(f_t(x), Y(t)) \mid X=x]$ where $L(y,y') = (y-y')^2$. Assume that there exists a constant $B > 0$ such that, $\ell_{f_t}(x)/B \in \cL$, where $\cL$ is a reproducing kernel Hilbert space of a kernel, $k$ such that $k(x, x) < \infty$. Finally, let $w_t : \cX \rightarrow \mathbb{R}_+$ be a valid re-weighting of $p_{t}$, $\E[w(X)\mid T=t] = 1$, and let $\tilde{w}(x) = \pi_t + (1-\pi_t)w(x)$, where $\pi_t = p(T=t)$. With probability at least $1-2\delta$,
  \begin{eqnarray*}
  R^{\vphantom{w_t}}(f_t) & \leq & \hat{R}^{\tilde{w}}_{t}(f_t) + B (1-\pi_t)\ipm_{\cL}(\hat{p}_{t}^{w_t}, \hat{p}^{\vphantom{w_t}}_{1-t}) \\
   & + & V_{p_t}(\tilde{w}, \ell_{f_t})\frac{\cC_{n_t,\delta}^{\cH}}{n_t^{3/8}} + \cD^{\cL}_{n_0,n_1,\delta}\left(\frac{1}{\sqrt{n_0}} + \frac{1}{\sqrt{n_1}}\right)
   + \sigma^2_{Y(t)}
 \end{eqnarray*}
  where $\cC_{n_t,\delta}^{\cH}$ is a function of the pseudo-dimension of $\cH$, {\blockedit $\cD^{\cL}_{n_0,n_1,\delta}$ is a function of the kernel norm of $\cL$ (see Lemma~\ref{lem:mmd_approx})}, both only with logarithmic dependence on $n$, $\sigma^2_{Y(t)}$ is the expected variance in $Y(t)$, and
  $
  V_{p_t}(w_t, \ell_{f_t}) = \max\left(\sqrt{\E_{p_t}[\tilde{w_t}^2\ell_{f_t}^2]}, \sqrt{\E_{\hat{p_t}}[\tilde{w_t}^2\ell_{f_t}^2]}\right)
  $.
  A similar bound exists where $\cL$ is the family of functions with Lipschitz constant at most 1 and $\mipm_{\cL}$ the Wasserstein distance, but with worse sample complexity. 
\label{thm:main}
\end{thmthm}
\end{mdframed}
\begin{proof}
  The result follows from application Lemmas~\ref{lem:cortes_bound}--\ref{lem:mmd_approx} to Lemma~\ref{lem:ipm_bound} and is given in larger generality in Theorem~\ref{thm:main_rep}. The $n^{3/8}$ rate may be improved at the cost of a more complicated expression, see discussion following Theorem 3 in~\citep{cortes2010learning}.
\end{proof}

{\blockedit 
\begin{thmrem}%
The spaces of hypotheses $\cH$ and the space of loss functions $\cL$ appear in Theorem~\ref{thm:main} through two different measures of complexity, the pseudo-dimension of $\cH$ and the kernel norm of losses in $\cL$. It is worth pointing out that $\cL$, which covers the loss functions of the hypotheses in $\cH$, may be larger than $\cH$, while the bound remains informative. For example, in experiments, we let $\cH$ be a family of fixed-size deep neural networks and $\cL$ the functions of an RKHS given by an RBF kernel of a given bandwidth. The pseudo-dimension of such an $\cH$, and therefore $\cC^\cH$, is bounded. The second complexity term, $\cD^\cL$, is bounded too---irrespective of whether the pseudo dimension of $\cL$ is bounded.
\end{thmrem}
}

We can also immediately state the following corollary.

\begin{thmcol}\label{col:main}
Assume that the conditions of Theorem~\ref{thm:main} hold. Let $f(x, t) := f_t(x)$ and let $\hR^w_p(f) := \sum_{i=1}^n w(x_i, t_i) L(f(x_i, t_i), y_i)/n$ represent the weighted empirical factual risk. Then, with $\tilde{w}(x, t) \coloneqq w_t(x) / \pi_t$, $n_{\textnormal{min}} = \min(n_0, n_1)$ and $\sigma^2 = \max(\sigma^2_{Y(0)},\sigma^2_{Y(1)})$ there is a term $K_{\cL, \cH, w, \delta, n_0, n_1}$ with at most logarithmic dependence on $n_0, n_1$,  such that
  $$
  \frac{R(\htau)}{2} \leq  \hat{R}^{\tilde{w}_t}_{p}(f)
   + B \left[ \pi_0\ipm_{\cL}(\hat{p}_0^{\vphantom{w_t}}, \hat{p}_1^{w_1}) + \pi_1 \ipm_{\cL}(\hat{p}_1^{\vphantom{w_t}}, \hat{p}_0^{w_0})\right]
   + \frac{K_{\cL, \cH, w, \delta, n_0, n_1}}{n_{\textnormal{min}}^{3/8}}
   + 2\sigma^2~.
 $$
 A tighter result may be obtained by decomposing the constant $K$. 
\end{thmcol}

Theorem~\ref{thm:main} and Corollary~\ref{col:main} hint at several interesting dependencies between generalization error, treatment group imbalance, sample re-weighting schemes and the choice of hypothesis class. We comment on these below.

\paragraph{Bounds with overlap and known propensity scores.} If overlap is satisfied and propensity scores known, applying importance weights $w_t(x) = p(T=t\mid X=x)/p(T=1-t\mid X=x)$ in Theorem~\ref{thm:main} leads to a tight bound in the limit of infinite samples (IPM and variance terms approach zero, re-weighted risk approaches desired population risk). A special case of this is the randomized controlled trial (RCT), in which $T$ has no dependence on $X$. In this setting, the IPM-terms depend only on the finite-sample differences between treatment groups---which may still be useful to characterize. It has been shown that under overlap, in the asymptotic limit, the best achievable sample complexity is unrelated to the imbalance of $p_0$ and $p_1$~\citep{alaa2018limits}. However, this setting is not our main concern as we are specification agnostic and focus on the finite-sample case.

{\blockedit 
\paragraph{Bounds without overlap.} Theorem~\ref{thm:main} \emph{does not} rely on treatment group overlap. Instead, it relies on an assumption that the true loss (w.r.t. features $X$ and potential outcome $Y(t)$) is a function in the given family $\cL$. The bound requires that the weights used for re-weighting, $w_t(x)$, are defined everywhere on $p_t(x)$ for $t\in {0,1}$. Importantly, if for some $x \in \cX$, $p_{1-t}(x)=0$ and $p_{t}(x)>0$, the density ratio $p_{t}(x)/p_{1-t}(x)$ is not defined and is not admissible as weighting function. Hence, the re-weighted densities cannot be equal in such regions and the IPM will be non-zero in this case. As a result, the bound cannot be tight in the general case without overlap, but may be informative if the lack of overlap is small. We return to the question of overlap in the next section, following Theorem~\ref{thm:main_rep}.
}

\paragraph{Bias and variance.} The term $V_{p_t}(\tilde{w}, \ell_{f_t})$ in Theorem~\ref{thm:main} shows that a less uniform re-weighting $w$ leads to larger variance (dependence on $n$). However, if $p_0$ and $p_1$ are very different, a non-uniform (balancing) $w$ is required to ensure unbiasedness, e.g., by making $p_t^{w_t} = p_{1-t}$. This indicates that $w$ introduces a bias-variance trade-off on top of the one typical for supervised learning. In particular, even if the true treatment propensity $\eta$ is known, a biased weighting scheme may lead to a smaller bound on the population risk when $p_0$ and $p_1$ are far apart.

\paragraph{Imbalance in non-confounders.} The size of the bounds in Theorem~\ref{thm:main} and Corollary~\ref{col:main} clearly depends on the quality of the hypothesis $f$ and the choice of re-weighting $w$. In addition, the IPM terms depend heavily on the input space $\cX$. In particular, if variables included in $\cX$ are predictive of $T$ but not predictive of $Y$, e.g., if they are instrumental variables~\citep{ding2017instrumental}, they will contribute to the IPM term but not to the expected risk, loosening the bound needlessly. If we can learn to ignore such information, we may obtain a tighter bound. To this end, in the next section, we derive bounds for representations $\Phi(X)$ of the original feature space.

\paragraph{Generalization under policy and domain shift.}
Predicting the conditional treatment effect for an individual may be viewed as predicting the effect of a change in treatment policy from one alternative to another. This notion may be generalized further by considering the estimation of treatment effects for change in policy on a population that differs from the one learned from. Specifically, this would involve a change not only in $p(T\mid X)$ but in  $p(X)$ as well. We studied this extended problem in~\citet{johansson2018learning}, and referred to changes in both policy $p(T\mid X)$ and domain $p(X)$ as a change in \emph{design}. We do not cover this setting in detail here.

%
%

%
%
\section{Generalization bounds for representation learning}
\label{sec:rep_bound}
When the input space $\cX$ increases in dimension, treatment groups $p_t(X)$ tend to grow increasingly different~\citep{d2017overlap} and, in general, this will lead to a looser bound in Theorem~\ref{thm:main}. To some extent, this can be mitigated by appropriately chosen weights $w$, but the additional finite-sample variance introduced by highly non-uniform weights may prevent tightening of the bound. In this section, we introduce another tool for minimizing bounds on the marginal risk in hypotheses that act on learned (potentially lower-dimensional) representations of the covariates $X$, see Figure~\ref{fig:rep} for an illustration. This allows hypotheses to focus their modeling power on particular aspects of the covariate space, ignoring others.

\begin{figure}
    \centering
    \includegraphics[width=.98\textwidth]{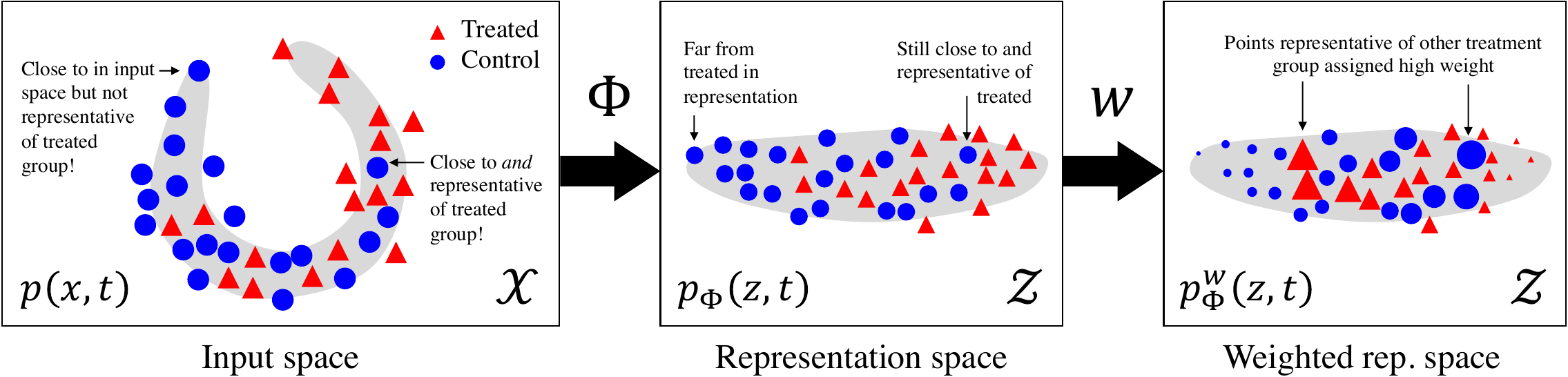}
    \caption{{\blockedit Treatment group imbalance in the observed context $X$ (left). Treated and controls have different distributions. Mapping contexts $x \in \cX$ to a representation space $\cZ$ using an embedding function $\Phi$ (middle) can allow for more efficient comparison of treatment groups and re-weighting (right) of observations to achieve a more balanced sample. In Section~\ref{sec:rep_bound}, we give generalization bounds for learning such representations and in Section~\ref{sec:experiments}, we use neural networks to fit them.}}
    \label{fig:rep}
\end{figure}

In many applications, the input distribution $p(X)$ is believed to be a low-dimensional manifold embedded in a high-dimensional space $\cX$, for example, the space of portraits embedded in the pixels of a photograph. In such settings, the best hypotheses are often simple functions of low-dimensional representations $\Phi(X)$ of the input~\citep{bengio2013representation}. The most famous examples of this are image and speech recognition for which representation learning using convolutional and recurrent neural networks advanced each field tremendously in only a few years~\citep{lecun2015deep}.

Let $\cE \subset \{\cX \rightarrow \cZ\}$ denote a family of representation functions (embeddings) of the input space $\cX$ into a space $\cZ \subset \bbR^v$ and let $\Phi \in \cE$ denote such an embedding function. \edit{For all theoretical results, we will assume that $\Phi$ is twice-differentiable and invertible.} Further, let $\cG \subseteq \{h : \cZ \rightarrow \cY\}$ denote a set of hypotheses $h(\Phi)$ operating on the representation $\Phi$ and let $\cH$ be the space of all such compositions, $\cH = \{f = h \circ \Phi : h \in \cG, \Phi \in \cE\}$.  Generalizing our discussion up to this point, we consider learning of $\Phi(X)$ from data with the goal of minimizing the marginal risk of hypotheses $h_t \circ \Phi$ for treatments $t \in \{0,1\}$.

For CATE to be identifiable from observations of $p(\Phi(X), T, Y)$, we need precisely the same requirements on $\Phi$ as previously on $X$, \emph{ignorability} and \emph{overlap},
\begin{equation}\label{eq:asmp_rep}
\forall t\in \{0,1\} : \underbrace{Y(t) \indep T \mid \Phi(X)}_{\textnormal{Ignorability}} \;\; \textnormal{ and } \;\; \forall z \in \cZ : \underbrace{p(T=t \mid \Phi(X) = z) > 0}_{\textnormal{Overlap}}~.
\end{equation}

Verifying the assumptions in \eqref{eq:asmp_rep} for a given $\Phi$, based on observational data alone, is impossible, just as for $X$. To address this, we consider learning twice-differentiable, invertible representations $\Phi : \cX \rightarrow \cZ$ where $\Psi : \cZ \rightarrow \cX$ is the \emph{inverse} representation, such that $\Psi(\Phi(x)) = x$ for all $x$. For treatment groups $t \in \{0,1\}$, we let $p_{\Phi,t}(z)$ be the distribution induced by $\Phi$ over $\cZ$, with $p^w_{\Phi,t}(z) := p_{\Phi,t}(z)w(\Psi(z))$ its re-weighted form and $\hat{p}^w_{\Phi,t}$ its re-weighted empirical form, following our previous notation. If $\Phi$ is invertible, ignorability and overlap in $X$ implies ignorability and overlap in $\Phi(X)$, as $p(\Phi(X)=z) = p(X=\Psi(z))$.

Building on Section~\ref{sec:theory}, we can now relate the expected marginal risk $R(h_t \circ \Phi)$ to the (expected) re-weighted factual risk $R^w(h_t \circ \Phi)$.
%
%

\begin{mdframed}[innerbottommargin=.8em,innertopmargin=0em]%
\begin{thmlem}\label{lem:exp_bound_rep}
  Suppose that $\Phi$ is a twice-differentiable, invertible representation, that $h_t(\Phi) \in \cG$ is a hypothesis, and that $f_t(x) = h_t \in \cH$ for $t\in \{0,1\}$. Let $\ell_{\Phi,h_t}(z) := \E_Y[L(h_t(z), Y(t))\mid X=\Psi(z)]$ be the expected pointwise loss given a representation $z$, where $L(y,y') = (y-y')^2$. {\blockedit Let $C_\Phi$ be a constant such that $\ell_{\Phi,h_t}\cdot |J_{\Psi}|/C_\Phi \in \cL$ with $\cL \subseteq \{\ell : \cZ \rightarrow \mathbb{R}_+\}$  where $J_{\Psi}(z)$ is the  Jacobian determinant of the representation inverse $\Psi=\Phi^{-1}$}. Then, with $\pi_t = p(T=t)$ and $w$ a valid re-weighting of $p_{\Phi,t}$, with $\tilde{w}$ defined as in Theorem~\ref{thm:main}, $\tilde{w}(x) = \pi_t + (1-\pi_t)w(x)$, s
  \begin{align}\label{eq:thm_main_exp_rep}
  R(f_t) & \leq \pi_t R^{\tilde{w}}_t(f_t) + (1 - \pi_t)\; C_\Phi \cdot \text{\emph{IPM}}_{\cL}(p_{\Phi,1-t}^{\vphantom{w}}, p_{\Phi,t}^{w})~.
  \end{align}
\end{thmlem}%
\end{mdframed}
\begin{proof}
By \eqref{eq:RpvsRt}, with $\ell_{f_t}(x) := \E_Y[L(f_t(x), Y(t)) \mid X=x]$, we have that
\begin{eqnarray*}
R(f_t) & = & \pi_t R_t^{\tilde{w}}(f_t)
 +  (1-\pi_t)\int_{x \in \cX}  \ell_{f_t}(x) \left(p_{1-t}(x) - p_t^{w}(x) \right)dx~.
\end{eqnarray*}
Then, by the standard change of variables, assuming that $\Phi$ is invertible, we have
{\blockedit 
\begin{eqnarray*}
\int_{x \in \cX} \ell_{f_t}(x) \left(p_{1-t}(x) - p_t^{w}(x) \right)dx
& = & \int_{z \in \cZ}  \left(p_{\Phi,1-t}(z) - p_{\Phi,t}^{w}(z) \right) \ell_{f_t}(\Psi(z))|J_{\Psi}(z)| dz \\
& = & \int_{z \in \cZ}  \left(p_{\Phi,1-t}(z) - p_{\Phi,t}^{w}(z) \right) \ell_{\Phi,h_t}(z)|J_{\Psi}(z)| dz \\
& \leq & C_\Phi \cdot \sup_{\ell \in \cL} \int_{z \in \cZ} \ell(z) \left(p_{\Phi,1-t}(z) - p_{\Phi,t}^{w}(z) \right) dz~.  \\
& = & C_\Phi \cdot \text{IPM}_{\cL}(p_{\Phi,1-t}^{\vphantom{w}}, p_{\Phi,t}^{w})~.
\end{eqnarray*}}%
Here, we have used the fact that, \edit{for} invertible $\Phi$, $p_{\Phi}(Z = \Phi(x)) = p(X=x)$.
\end{proof}

\paragraph{The scale of $\Phi$ and the factor $C_\Phi$.}
Comparing Lemma~\ref{lem:exp_bound_rep} (bound in representation) to Corollary~\ref{col:ipm_bound} (original space), we notice two immediate differences: the additional factor $C_\Phi$ and the change from measuring distributional distance in $X$ to doing so in $Z$, via $\Phi$. The most illustrative example for why $C_\Phi$ is necessary is when $\Phi$ simply reduces the scale of $X$, i.e., when $\Phi(x) = x/a$ for $a>1$. IPMs often vary with the scale of the space in which they are applied and we could reduce the right-hand side of the bound simply by scaling down $X$ were it not for $C_\Phi$, which counteracts this reduction.

\paragraph{The influence of $\Phi$ on the IPM}
Measuring distributional distance in $\Phi$ with a fixed IPM family $\cL$ means that we may emphasize or de-emphasize part of the covariate space, even when $\Phi$ is invertible. For example, if $\Phi$ is a linear function that scales down a component $X_d$ of $X$ substantially, and $\cL$ is a family of linear functions with bounded norm, the influence of distributional differences in $X(d)$ on the IPM is reduced.

With Lemma~\ref{lem:exp_bound_rep} in place, we can now state the result for the finite-sample case by following the same steps as in Section~\ref{sec:theory}.

\begin{mdframed}[innerbottommargin=.8em,innertopmargin=0em]%
\begin{thmthm}\label{thm:main_rep}
  Given is a sample $(x_1, t_1, y_1), ..., (x_n, t_n, y_n) \overset{i.i.d.}{\sim} p(X, T, Y)$ with empirical measure $\hat{p}$. Assume that ignorability (Assumption~\ref{asmp:ignorability}) holds w.r.t. $X$. Suppose that $\Phi$ is a twice-differentiable, invertible representation, that $h_t(\Phi)$ is a hypothesis on $\cZ$, and $f_t = h_t(\Phi(x)) \in \cH$. Let $\ell_{\Phi, h_t}(z) := \E_Y[L(h_t(z), Y(t)) \mid X=\Psi(z)]$ where $L(y,y') = (y-y')^2$. Further, let $A_\Phi$ be a constant such that
  $\forall z\in \cZ : A_\Phi \geq |J_{\Psi}(z)|$, where $J_{\Psi}(z)$ is the Jacobian of the representation inverse $\Psi$, and assume that there exists a constant $B_\Phi > 0$ such that, with $C_\Phi := A_\Phi B_\Phi$,  $\ell_{\Phi,h_t}/C_\Phi \in \cL$, where $\cL$ is a reproducing kernel Hilbert space of a kernel, $k$ such that $k(x, x) < \infty$. Finally, let $w$ be a valid re-weighting of $p_{\Phi, t}$. Then, with probability at least $1-2\delta$,
  {\blockedit
  \begin{align}
  R_{1-t}^{\vphantom{w}}(f_t) & \leq \hat{R}^{w}_t(f_t) + C_\Phi\cdot \text{\emph{IPM}}_{\cL}(\hat{p}_{\Phi, 1-t}, \hat{p}_{\Phi, t}^{w}) \nonumber \\
   & + V_{p_t}(w, \ell_{f_t})\frac{\cC_{n_t,\delta}^{\cH}}{n_t^{3/8}} + \cD^{\Phi,\cL}_{n_0,n_1,\delta}\left(\frac{1}{\sqrt{n_0}} + \frac{1}{\sqrt{n_1}}\right)
   + \sigma^2_{Y(t)}%
   \label{eq:thm_rep_bound}
 \end{align}
 }
  where $\cC_{n,\delta}^{\cH}$ is a function of the pseudo-dimension of $\cH$, {\blockedit $\cD^{\cL}_{n_0,n_1,\delta}$ is a function of the kernel norm of $\cL$ (see Lemma~\ref{lem:mmd_approx})}, both only with logarithmic dependence on $n$ and $m$, $\sigma^2_{Y(t)}$ is the expected variance in $Y(t)$, and
  $
  V_p(w, \ell_{f}) = \max\left(\sqrt{\bbE_{p}[w^2\ell_{f}^2]}, \sqrt{\bbE_{\hat{p}}[w^2\ell_{f}^2]}\right)
  $.
  A similar bound exists where $\cL$ is the family of functions Lipschitz constant at most 1 and $\mipm_{\cL}$ the Wasserstein distance, but with worse sample complexity.
\end{thmthm}
\end{mdframed}
\begin{proof}
  The result follows from application Lemmas~\ref{lem:cortes_bound}--\ref{lem:mmd_approx} to Lemma~\ref{lem:exp_bound_rep}. The $n^{3/8}$ rate may be improved at the cost of a more complicated expression, see discussion following Theorem 3 in~\citep{cortes2010learning}.
\end{proof}

\paragraph{Overlap, ignorability and invertibility.} Theorem~\ref{thm:main_rep} holds both with and without treatment group overlap in $X$. It is important to note, however, that when we change the covariate space from $X$ to $\Phi$, the assumption that $\ell_{\Phi,h_t}/C_\Phi \in \cL$ is not guaranteed, even for large $C_\Phi$, since information on which $\ell$ depends may have been (approximately) removed. In the context of risk minimization, information is only excluded from $\Phi$ if it is not predictive of the outcome $Y$, in which case it is independent also of $\ell$. Thus, under the additional assumption of overlap, the assumption that $\ell_{\Phi,h_t}/C_\Phi \in \cL$ is verifiable in the limit of infinite data. In~\citet{johansson2019support}, we expand on the effects of non-invertibility on identifiability of the marginal risk in much greater detail. In particular, we show that for non-invertible $\Phi$, without overlap, the marginal risk may be bounded under the assumption that information removed in $\Phi$ is as important to the risk of the factual outcome as to that of the counterfactual. This assumption, however, is also unverifiable in general.

%
%
\subsection{Relation to unsupervised domain adaptation}
\label{sec:domainadaptation}
Connections between the problem of estimating causal effects and learning under distributional shift have been pointed out in several contexts~\citep{tian2001causal,zhang2013domain}. In particular, \citet{johansson2016learning} showed that estimating counterfactual outcomes under ignorability is mathematically equivalent to unsupervised domain adaptation between domains $D\in \{0,1\}$ under covariate shift. We make this connection precise below.
\renewcommand\arraystretch{1.2}
\begin{table}[ht!]
  \centering
  \begin{tabular}{llll}
    \toprule
    {\bf Task} & {\bf Data} & {\bf Goal} & {\bf Assumption} \\
    \midrule
    \multirow{2}{*}{Causal estimation} &
    Factual & Counterfactual & Ignorability\\
    & $(x, t, y) \sim p(X, T, Y)$ & $p(Y(1-T) \mid X, T)$ & $Y(t) \indep T \mid X$ \\
    \midrule
    \multirow{2}{*}{Domain adaptation} &
    Source domain & Target label & Covariate shift\\
    & $(x, y) \sim p(X, Y \mid D=0)$ & $p(Y \mid X, D=1)$ & $Y \indep D \mid X$ \\
    \bottomrule
  \end{tabular}
\end{table}
\renewcommand\arraystretch{1}

The bounds we present in this work are related to a series of work on generalization theory for unsupervised domain adaptation~\citep{ben2007analysis,mansour2009domain,long2015learning}, but differ in significant ways. Superficially, the bounds given in these papers have a similar form using the sum of observed risk in the source domain and distributional distance w.r.t. a function class $\cH$ to bound the risk in the target domain:
$$
R_{D=1}(f) \leq R_{D=0}(f) + d_\cH(p(X\mid D=1), p(X\mid D=0)) + \lambda_\cH.
$$
Similarly, these bounds do not rely on overlap but cannot guarantee consistent estimation in the general case. In fact, because they do not allow for re-weighting of domains, even when source and target domains completely overlap, these bounds are often unnecessarily loose. Furthermore, while they are used to motivate representation learning algorithms, these bounds do not apply to learned representations without modification~\citep{johansson2019support}. In this work, we overcome this issue by requiring that representations $\Phi$ are invertible.

%
%

\section{Estimation}
\label{sec:estimation}
We turn now to deriving practical algorithms inspired by our theoretical results. First, we give learning objectives for estimating potential outcomes and CATE grounded in the theoretical results of Sections~\ref{sec:theory}--\ref{sec:rep_bound}, {\blockedit and show that they lead to asymptotically consistent estimation.  Then, we describe possible parameterizations of representation functions $\Phi$, discuss possible drawbacks of estimating  potential outcomes separately, and introduce a neural network architecture and objective for shared representation learning between treatment groups. While these parameterizations are not always guaranteed to minimize our bounds, they benefit from the insights gathered from them, as shown in empirical results.
}

%
%
\subsection{A first learning objective and asymptotic consistency}
\label{sec:objective}
\label{sec:learningweights}
\label{sec:consistency}

Let $D = \{(x_1, t_1, y_1), ..., (x_n, t_n, y_n)\}$ be a set of samples drawn i.i.d. from $p(X, T, Y)$ and let each sample $i$ be endowed with weights $w_i = w(x_i, t_i)$ for some function $w : \cX \times \{0,1\} \rightarrow \mathbb{R}_+$. Let $n_t = \sum_{i=1}^n \mathds{1}[t_i=t]$ be the number of samples with treatment assignment $t$. Further, let $\lambda, \alpha > 0$ be hyperparameters controlling the strength of the regularization of functional complexity, as measured by $\mathcal{R}$, and distributional distance $\ipm_{\cL}$ respectively. Recall that $\hat{p}^w_{\Phi,t}$ is the re-weighted factual distribution of representations $\Phi$ under $p_t$. %
{\blockedit Now, we consider using compositions $f_t = (h_t \circ \Phi) \in \cH$ of hypotheses $h_t \in \cG \subseteq \{\cZ \rightarrow \cY\}$ and representations $\Phi \in \cE \subseteq \{\cX \rightarrow \cZ\}$ to estimate a single expected potential outcome $\E[Y(t) \mid X]$. Then, directly motivated by Theorems~\ref{thm:main}--\ref{thm:main_rep}, we propose to minimize the following learning objective, with hyperparameters $\beta = (\lambda, \alpha)$, 
\begin{align}\label{eq:emp_loss}
\cO_t(h_t, \Phi; \beta) = \underbrace{\sum_{i:t_i=t} \frac{\tilde{w}_i}{n_t} L(h_t(\Phi(x_i)), y_i)}_{\text{Weighted factual risk}}
\;\; + \underbrace{\frac{\lambda}{\sqrt{n_t}} \cR(h_t) \sumphantom}_{\text{Regularization}}
 + \underbrace{\alpha\pi_{1-t}\ \ipm_{\cL}(\hat{p}_{\Phi, t}^{w}, \hat{p}_{\Phi, 1-t}^{\vphantom{w}}) \sumphantom}_{\text{Distributional \edit{regularization}}}
\end{align}
where $\tilde{w}_i = \pi_t + \pi_{1-t} w_i$ and $\pi_t = p(T=t)$.

Under Assumptions \ref{asmp:ignorability} (ignorability) and \ref{asmp:overlap} (overlap), for balancing weights $w_i = p(T=t_i)/p(T=t_i \mid X=x_i)$, objective \eqref{eq:emp_loss} reduces to inverse propensity-weighted regression in the limit of infinite samples~\citep{freedman2008weighting}. In the finite-sample regime, the IPM does not vanish even if $p_0 = p_1$, because of sample variance. 

As pointed out previously, the results of Theorems~\ref{thm:main}--\ref{thm:main_rep} remain upper bounds on the CATE risk for  weights other than the inverse propensity score. In fact, as shown by \citet{hirano2003efficient}, weighting using the true propensity score does not always lead to the most efficient estimation. Thus, in addition to learning representations and hypotheses, we may consider learning the sample weights $w$ jointly, controlling the variance introduced by non-uniform weights by regularizing the norm of $w$~\citep{johansson2018learning}. With $\beta=(\alpha, \lambda_h, \lambda_w)$ a set of hyperparameters, we consider the following objective,
{ \everymath={\displaystyle}
\begin{align}\label{eq:emp_loss_w}
\begin{array}{ll}
\cO_t(h_t, \Phi, w^t; \beta) & = \sum_{i:t_i=t} \frac{\tilde{w}^t_i}{n_t} L(h_t(\Phi(x_i)), y_i) +  \frac{\lambda_h}{ \sqrt{n_t}}\cR(h_t) 
\\
 & + \vphantom{\sum_i^n} \alpha\pi_{1-t}\ \ipm_{\cL}(\hat{p}_{\Phi, t}^{w^t}, \hat{p}_{\Phi, 1-t}^{\vphantom{w^t}}) +  \frac{\lambda_{w}}{ n_t}\|w^t\|_2 
 \end{array}%
\end{align}%
}%
with $\tilde{w}^t_i = \pi_t + \pi_{1-t} w^t_i$ and subject to the constraint $\sum_i w^t_i=n_t$.
}

A theoretical advantage of using objective~\eqref{eq:emp_loss_w} is that it allows for an explicit tradeoff between bias and variance induced by the sample weights. \edit{In principle, the weighting function can be implemented using one free parameter per sample point. However, our formulation allows the weights to be defined in terms of the learned representation $\Phi$ rather than free parameters or as a function of $X$}. This imposes further restrictions on the variance in and induced by $w$. We proceed to give conditions under which minimization of $\eqref{eq:emp_loss}$ leads to consistent estimation of conditional expected potential outcomes.

%
%
%
\begin{mdframed}[innerbottommargin=.8em,innertopmargin=0em]%
\begin{thmthm}\label{thm:asymptotics}
Suppose $\cH$ is a reproducing kernel Hilbert space (RKHS) given
by a bounded kernel $k$, such that for all $h_t\in \cG, \Phi\in \cE$, $h_t \circ \Phi \in \cH$.
Suppose weak overlap holds in that 
$$\forall t \in \{0,1\} : \E_X[(p_t(X)/p_{1-t}(X))^2] < \infty~.
$$
Let $\cO_t$ be the objective defined in \eqref{eq:emp_loss_w}, \edit{with the IPM-space $\cL$ also an RKHS with bounded kernel $k_\cL$}, and $n_t = \sum_{i=1}^n\mathds{1}[t_i = t]$ for $t\in \{0,1\}$. Then, 
$$
\min_{h_t,\Phi,w^t}\cO_t(h_t, \Phi, w^t; \beta) \leq \min_{f_t \in \cH}R(f_t) + O_p(1/\sqrt{n_0}+1/\sqrt{n_1}) ~,
$$
\edit{where $O_p$ denotes stochastic boundedness.}
Thus, under the assumptions of Thm.~\ref{thm:main_rep}, for sufficiently large $\alpha$ and $\lambda_w$, with $\hat{f}^n$ the minimizer of \eqref{eq:emp_loss_w} for $n$ samples,
$$
R(\hat f_t^n)\leq \min_{f_t \in \cH}R(f_t) + O_p(1/n_0^{3/8}+1/n_1^{3/8}).
$$
In words, the minimizers of $\eqref{eq:emp_loss_w}$ converge to the representation and hypothesis that minimize the counterfactual risk, in the limit of infinite samples.
\end{thmthm}
\end{mdframed}
Thorem~\ref{thm:asymptotics} is proven in Appendix~\ref{app:asymptotics}. The result may be generalized further to the case of estimating the value of an abritrary treatment policy under a shift in marginal distribution $p(X)$. See the Appendix and \citet{johansson2018learning} for results in this setting.

{
\blockedit
\paragraph{Relating bound and learning objective.}
Learning objectives \eqref{eq:emp_loss}--\eqref{eq:emp_loss_w} are direct results of Theorem~\ref{thm:main_rep}. However, unlike the factor $C_\Phi$ in \eqref{eq:thm_rep_bound}, the hyperparameter $\alpha$, which takes its place in the objectives, is independent of the representation $\Phi$. As remarked previously, when $\Phi$ shrinks the representation space, $C_\Phi$ must grow for Theorem~\ref{thm:main_rep} to be valid. Hence, any fixed $\alpha$ might be either too small during the learning of $\Phi$, implying that the learning objective is not a bound on the counterfactual risk, or too large, over-penalizing distributional distance (see Figure~\ref{fig:ihdp_vs_alpha}). We view the role of $\alpha$ similar to that of the regularization parameters of LASSO and ridge regression---not for directly evaluating a generalization bound, but for controlling the tradeoff between bias and variance~\citep{friedman2001elements}. In empirical results, we find that a fixed value of $\alpha$ can help learning an estimator with lower counterfactual risk. 
}

%
%
\subsection{\edit{Practical} representation learning \edit{and learning shared representations}}
\label{sec:models}

{\blockedit 
A well-selected or learned representation $\Phi$ can lead to estimates of potential outcomes and CATE with improved counterfactual risk. The two most prominent approaches to learning representations in the literature are i) deep neural networks~\citep{bengio2013representation} and ii) variable selection~\citep{schneeweiss2009high}. As non-trivial examples of the latter cannot satisfy the assumption of invertibility of $\Phi$ even approximately, we restrict our attention to parameterizations of $\Phi$ as neural networks. Neural networks have been made invertible by design in the context of normalizing flows~\citep{kobyzev2020normalizing} and have been observed to be approximately invertible spontaneously~\citep{gilbert2017towards}. In experiments, we do not enforce or guarantee invertibility of our neural networks but nevertheless benefit from the insights in our theoretical results in regularizing treatment group distance in representations.%
}

\begin{figure}
  \centering
  \begin{subfigure}{.41\textwidth}
    \centering
    \includegraphics[height=3.8cm]{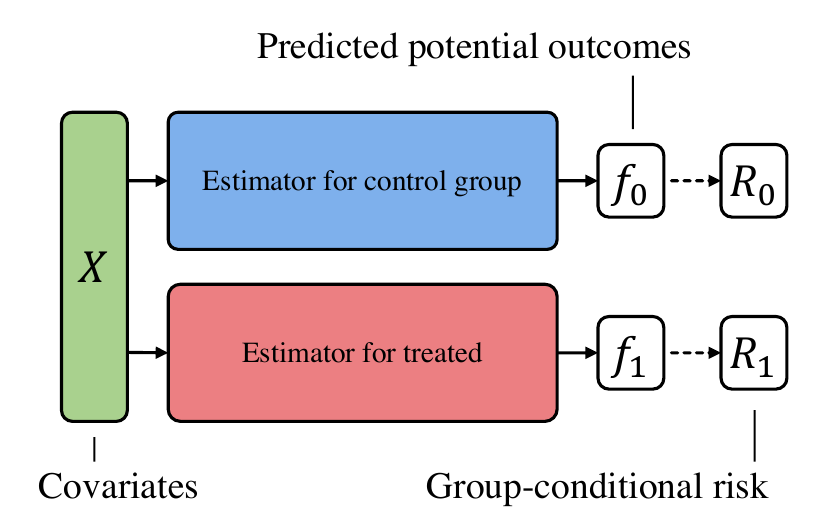}
    \caption{\label{fig:tlearner}T-learner}
  \end{subfigure}
  \hfill%
  \begin{subfigure}{.58\textwidth}
    \centering
    \includegraphics[height=3.8cm]{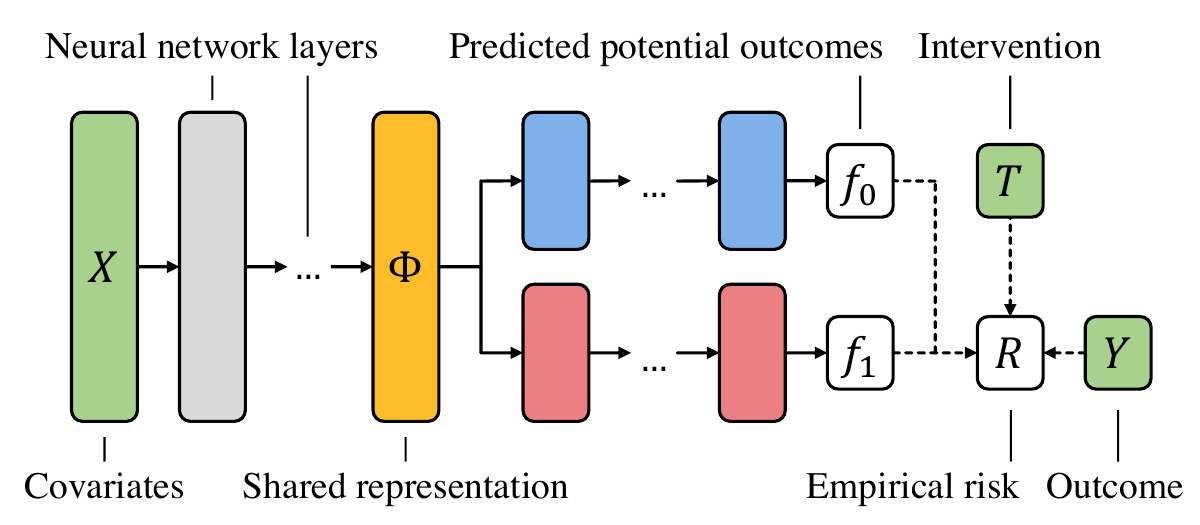}
    \caption{\label{fig:tarnet}TARNet~\citep{shalit2016estimating}}
  \end{subfigure}
  \caption{\label{fig:tlearnertarnet}Estimator architectures for potential outcomes and conditional average treatment effects. Green boxes indicate inputs, white boxes outputs and loss terms, yellow boxes shared representations and blue/red boxes estimators of potential outcomes. Solid lines indicate transformation part of the prediction function and dashed lines indicate computations part of the learning procedure.}
\end{figure}

{\blockedit
Parameterizing $\Phi$ with neural networks leaves a lot of freedom in the design of estimators. A straight-forward application minimizes the factual risk or objective~\eqref{eq:emp_loss_w} independently for each potential outcome, with separate neural network architectures. This approach is that of  T-learners, which fit two models, one for each treatment group, as described Section~\ref{sec:related}. A potential drawback of this approach is that in many cases, much of the variance in \emph{both} potential outcomes may be explained by the same patterns in the input~\citep{kunzel2017meta}. Estimating each outcome independently fails to take advantage of this fact.

We proposed a natural extension in the Treatment-Agnostic Representation Network (TARNet) in \citet{shalit2016estimating}. In TARNet, a T-learner architecture is appended to a representation $\Phi$ shared between treatment groups (see Figure~\ref{fig:tlearnertarnet} for a comparison with T-learners). TARNet has the advantage of sharing samples between treatment groups in learning the representation which may be useful when $\tau$ is a simpler function of $X$ than $Y(0), Y(1)$. It improves on classical T-learning estimators by allowing estimators of different potential outcomes to share information through representation functions learned from both treatment groups. TARNet minimizes the following objective,
\begin{align}\label{eq:obj_tarnet}
\cO_{\mathrm{TARNet}}(h, \Phi) =  \frac{1}{n} \sum_{i=1}^n L(h(\Phi(x_i), t_i), y_i) + \frac{\lambda_h}{\sqrt{n}} \cR(h)
\end{align}
By also regularizing treatment group distance in learned representations, we may enable better counterfactual generalization, in line with Theorem~\ref{thm:main_rep}. The CounterFactual Regression (CFR) estimator, introduced in \citet{shalit2016estimating} and illustrated in Figure~\ref{fig:cfr}, applies this idea to TARNet, minimizing the objective
\begin{align}\label{eq:obj_cfr}
\cO_{\mathrm{CFR}}(h, \Phi) = \frac{1}{n} \sum_{i=1}^n  L(h(\Phi(x_i), t_i), y_i) + \alpha \cdot \ipm_{\cL}(\hat{p}_{\Phi, 0}, \hat{p}_{\Phi, 1}) + \frac{\lambda_h}{\sqrt{n}} \cR(h) ~.%
\end{align}%
\begin{figure}%
\centering%
\includegraphics[height=.3\textwidth]{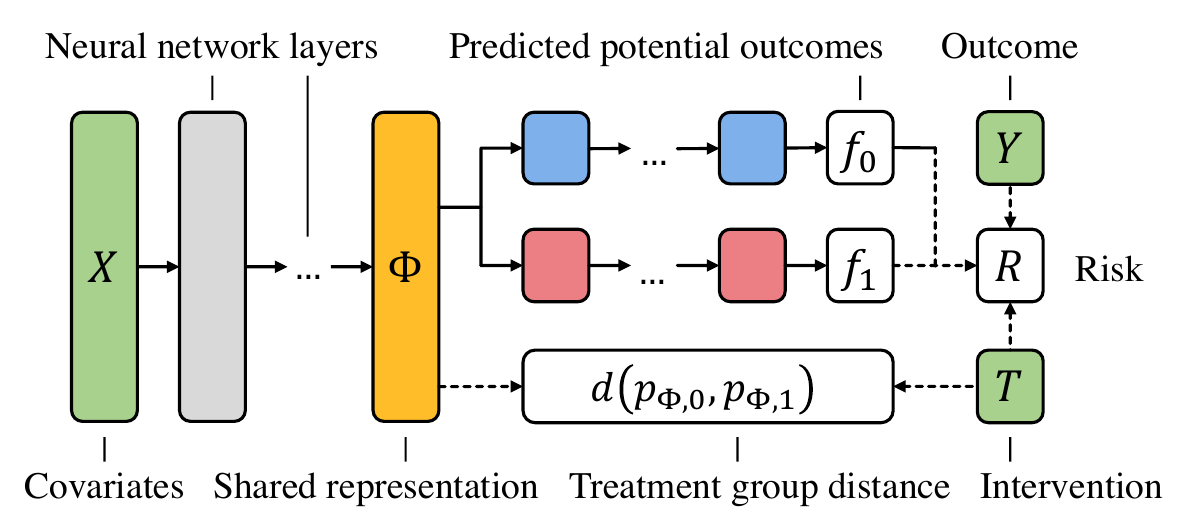}
\caption{\label{fig:cfr}Illustration of the Counterfactual Regression (CFR) estimator. Here, $d$ represents a distributional distance such as an IPM. The visual elements are described in Figure~\ref{fig:tlearnertarnet}.}
\end{figure}%
Finally, the re-weighted CFR (RCFR) minimizes a weighted objective analogous to objective~\eqref{eq:emp_loss_w} but with a shared representation~\citep{johansson2018learning}. 
\begin{align}\label{eq:obj_rcfr}
\cO_{\mathrm{RCFR}}(h, \Phi, w) = \sum_{i=1}^n \frac{w_i}{n} L(h(\Phi(x_i), t_i), y_i) + \alpha \cdot \ipm_{\cL}(\hat{p}^{w}_{\Phi, 0}, \hat{p}^{w}_{\Phi, 1}) + \frac{ \cR(h)}{\lambda_h^{-1} \sqrt{n}}  + \frac{\|w\|_2}{\lambda_{w}^{-1} n}
\end{align}
subject to the constraint $\sum_{i:t_i=t}w_i = n_t$ for $t\in \{0,1\}$.
The weights are given by a learned function, parameterized by a neural network, to avoid over-penalizing distributional distance and to trade off the effects of selection bias and variance. See Appendix~\ref{app:rcfr} for an illustration of the neural network architecture used in RCFR.
}
%
%
\subsection{Implementing regularization of distributional distance}
The idea of regularizing models to be invariant to changes in a variable, in our case the treatment indicator, is prevalent through-out machine learning~\citep{ganin2016domain,goodfellow2014generative,long2015learning}. As a result, several families of distance metrics between distributions have been used to impose such constraints. The most common of these are $f$-divergences (e.g. the KL-divergence)~\citep{nowozin2016f}, integral probability metrics (e.g., the maximum-mean discrepancy) and adversarial discriminators~\citep{ganin2016domain}. $f$-divergences are often ill-suited for comparing two empirical densities as they are based on the density ratio which is undefined in any point outside of the support of either density. In contrast, IPMs are based on the density difference which is defined everywhere. Adversarial methods are based on the metric implied by a learned discriminator function which is trained to distinguish samples from the two densities. The flexibility of this approach---that it tailors the metric to observed data---is also its weakness since optimization of adversarial discriminators is fraught with difficulty.
In implementations of CFR and RCFR, we use the empirical kernel MMD~\citep{gretton2012kernel} and the Wasserstein distance~\citep{villani2008optimal}, \edit{both examples of integral probability metrics}. 

\paragraph{Minimizing the empirical maximum mean discrepancy.} The maximum mean discrepancy (MMD) was popularized in machine learning through its kernel-based incarnation in which the associated function family is a reproducing kernel Hilberg space (RKHS)~\citep{gretton2012kernel}. We restrict our attention to this family here. An unbiased estimator of the MMD distance between densities $p, q$ on $\cX$, with respect to a kernel $k$, may be obtained from samples $x_1, ..., x_m \sim p$, $x_1', ..., x_n' \sim q$ as follows.
$$
\widehat{\text{MMD}}^2_k(p,q) := \frac{1}{m(m-1)}\sum_{i=1}^m\sum_{j\neq i}^m k(x_i, x_j) - \frac{2}{mn}\sum_{i=1}^m\sum_{j=1}^n k(x_i, x_j') + \frac{1}{n(n-1)}\sum_{i=1}^n\sum_{j\neq i}^n k(x_i', x_j')
$$
By choosing a differentiable kernel $k$, such as the Gaussian RBF-kernel, we can ensure that the MMD is amenable to gradient-based learning. In applications where the quadratic time complexity w.r.t. sample size is prohibitively large, another unbiased estimator (but with larger variance) may be obtained by sampling pairs of points $(x_1, x_1'), ..., (x_{2n}, x_{2n}') \sim p\times q$ and comparing only elements within pairs~\citep{long2015learning},
$$
\widehat{\text{MMD}}^2_k(p,q) := \frac{1}{n}\sum_{i=1}^n \left[ k(x_{2i-1}, x_{2i}) + k(x_{2i-1}', x_{2i}') - k(x_{2i-1}, x_{2i}') - k(x_{2i}, x_{2i-1}')\right]~.
$$

\paragraph{Minimizing the Wasserstein distance.} The Wasserstein distance is typically computed as the solution to a linear program (LP). The gradient of this solution with respect to the learned representation may be obtained through the KKT conditions of the problem and the solution for the current representation~\citep{amos2017optnet}. However, solving the LP at each gradient update is prohibitively expensive for many applications. Instead, we minimize an approximation of the distance known as Sinkhorn distances~\citep{cuturi2013sinkhorn}, computed using fixed-point iteration. In previous work~\citep{shalit2016estimating}, we computed the distance and its gradient by forward and backpropagation through a recurrent neural network with transition matrix corresponding to the fixed-point update. For a full description, see Appendix~\ref{app:wasserstein}. Alternative methods for minimizing Wasserstein distances have been developed in the context of generative adversarial networks (GANs)~\citep{arjovsky2017wasserstein}.

%
%

\section{Experiments}
\label{sec:experiments}

Evaluating estimates of potential outcomes and causal effects from observational data is notoriously difficult as ground-truth labels are hard or impossible to come by. Cross-validation and other sample splitting schemes frequently used to evaluate supervised learning are not immediately applicable to our setting for this reason. Moreover, the task of producing the labels themselves is exactly the task we are attempting to solve. As a result, estimation methods are often evaluated on synthetic or semi-synthetic data, where consistent estimation or computation of the labels are guaranteed. Another alternative is using real-world data where the treatment-assignment randomization is known, e.g. data from an RCT. In this section, we give a suite of experimental results on synthetic, semi-synthetic and real-world data. Our experiments are developed to separately highlight the impact of architecture choice and the balancing regularization scheme.

%
%
\subsection{Experimental setup \& baselines}
{\blockedit
We evaluate learning to predict potential outcomes using the TARNet and CFR objectives, \eqref{eq:obj_tarnet} and \eqref{eq:obj_cfr}.  Recall that TARNet is equivalent to CFR with the parameter $\alpha=0$. 
}
We specify the function family used in the IPM by a subscript, e.g., CFR$_{\textnormal{MMD}}$, and point out for which experiments the weighting function is learned and for which it is set to the uniform weighting. All variants of CFR were implemented as feed-forward neural networks with exponential-linear units and architectures as described in Section~\ref{sec:estimation}. \edit{These implementations do not guarantee the invertibility condition of our theoretical results but benefit from the regularization scheme derived from them. Further, in~\citet{johansson2019support,gilbert2017towards}, it was found that similar networks satsify invertibility approximately without explicitly being made to.} %
Ranges for hyperparameters, such as layer sizes, learning rates et cetera, are described in Appendix~\ref{app:exp} and specific values were selected according to a procedure below. An implementation of CFR with uniform sample weights may be found at \url{https://github.com/clinicalml/cfrnet}.

As our primary baseline, we use two variants of Ordinary Least Squares (linear regression). The first ({\sc OLS-S}) adopts the S-learner paradigm and includes the treatment variable $T$ as a feature in the regression. The second ({\sc OLS-T}) is a T-learner where the outcome in each treatment arm is modeled using a separate linear regression. Our other simple baseline is a $k$-nearest neigbor regression which imputes counterfactual outcomes of a unit by the average of its $k$-nearest neighbors with the opposite treatment assignment.

For a more challenging comparison, we use Targeted Maximum Likelihood, which is a doubly robust method ({\sc TMLE})~\citep{gruber2012tmle} which uses an ensemble of machine-learning methods. We also compare with a suite of tree-based estimators: First, we use a Random Forest ({\sc Rand. For.})~\citep{breiman2001random} in the S-learner paradigm by including $T$ as a feature. Second, we include tree-based methods specifically designed or adapted for causal effect estimation: Bayesian Additive Regression Trees ({\sc BART})~\citep{chipman2010bart,bayestree} and Causal Forests
({\sc Caus. For.})~\citep{wager2015estimation,causalforests}.
Finally, we also compare with our earlier work on
Balancing Linear Regression ({\sc BLR}) and Balancing Neural Network
({\sc BNN})~\citep{johansson2016learning}.

\subsubsection*{Evalutation criteria \& hyperparameter selection}
To assess the quality of CATE estimates, either knowledge of the propensity score or the outcome function is required. Where labels are available, our primary criterion for evaluation is the mean squared error in the imputed CATE as defined in \eqref{eq:taumse}. When only the propensity score is available, such as in a randomized controlled trial or other experiments, we instead estimate the \emph{policy risk} as defined below.

A policy $\pi$ is any (possibly stochastic) function that maps from covariates $x$ to treatment decision $t\in \{0,1\}$; we will only consider deterministic policies. The risk of a policy $\pi$ for outcomes $Y \in [0,1]$, where large $Y$ is considered beneficial, is
$$
R_{\mbox{\tiny Pol}}(\pi) := 1 - \E_X[E_{(Y(0),Y(1)}[ Y(\pi(x)) \mid X = x]]~.
$$
A good policy is one that for a given $x$ will choose the potential outcome with the higher conditional expectation given $x$.
If we know the true propensity scores $p^*(t_i=1|x_i)$ used in generating the dataset,
then the risk of a deterministic policy $R_{pol}(\pi)$ may be estimated using rejection sampling based on a sample $(x_1, t_1, y_1), \ldots,$ $(x_m, t_m, y_m)$ and propensity scores $p^*(t_1=1|x_1), \ldots, p^*(t_m=1|x_m)$ by considering only the \emph{propensity re-weighted effective sample} on which the proposed policy agrees with the observed one:
\begin{equation}
\hat{R}_{\mbox{\tiny Pol}}(\pi) := 1 - \frac{\sum_{i=1}^m y_i \mathds{1}[\pi(x_i) = t_i]p^*(t_i=\pi(x_i)|x_i)}{\sum_{i=1}^m \mathds{1}[\pi(x_i) = t_i]} ~.
\label{eq:policy_curve}
\end{equation}
A downside of this estimator is that it has very high variance for policies that are very different from the observed policy. Note that in the case where the data was generated by an RCT with equal probability of treatment and control, the propensity scores have a particularly simple form: $p(t=1|x) = 0.5$ for all $x$.

In our experiments, we evaluate the policy $\pi_f : \cX \rightarrow \cT$ induced by an estimator $f(x, t)$ of potential outcomes and a threshold $\lambda$ such that
\begin{equation}\label{eq:lambda_pol_risk}
\pi_f(x) :=
\left\{
\begin{array}{ll}
  1, & \mbox{if } f(x, 1) - f(x, 0) > \lambda \\
  0, & \mbox{otherwise }
\end{array}
\right. ~.
\end{equation}
By varying $\lambda$ from low to high we obtain a curve that interpolates between liberal and conservative allocation of treatment.

In all experiments we fit a model on a training set and then evaluate on a held-out set. We always report results both within-sample and out-of-sample. We wish to emphasize that the within-sample results should \emph{not} be thought of as training-loss in standard ML problems. Even within-sample results include the challenging task of inferring unobserved counterfactuals for the training samples.

\paragraph{Hyperparameter selection.}
We choose hyperparameters for all estimators in the same way. As the ground truth potential outcomes are unavailable to us, we use pseudo-labels for the true CATE imputed using a nearest-neighbor estimator. With $j(i)$ the nearest ``counterfactual'' neighbor of sample $i$ in Euclidean distance, such that $t_{j(i)} \neq t_i$, we define
$$
\widehat{\textnormal{MSE}}_{\textnormal{nn}}(f) := \frac{1}{n}\sum_{i=1}^n\left((1 - 2t_i)(y_{j(i)} - y_i) - (f(x_i, 1) - f(x_i, 0))\right)^2 %
$$
and use its value on a held-out validation set as a surrogate for the true MSE in $\htau$ in hyperparameter section.
This choice may bias selection of hyperparameters towards preferring models close to a nearest-neighbor estimator, but we anticipate this effect to be mild as  $\widehat{\textnormal{MSE}}_{\textnormal{nn}}$ is not used as a training objective. For neural network estimators, we perform early stopping based on the training objective evaluated on a held-out validation set in the IHDP study, and based on held-out policy risk in the Jobs study (both described below). Ranges for hyperparameters for CFR are presented in Appendix~\ref{app:exp}.

%
%
\subsection{Synthesized outcome: IHDP}
The Infant Health and Development Program (IHDP) dataset has been frequently used to evaluate machine learning approaches to causal effect estimation in recent years~\citep{hill2011bayesian}. The orginal data comes from a randomized study of the impact on educational and follow-up interventions on child cognitive development~\citep{brooks1992effects}. Each observation represents a single child in terms of 25 features of their birth and their mothers. To introduce confounding, \citet{hill2011bayesian} removed a biased subset of the treatment group---all treated children with nonwhite mothers---leaving 747 subjects in total. This induces not only confounding, but also lack of overlap in variables strongly correlated with race (race itself was removed from the feature set following the biased selection). To enable consistent evaluation, the outcome of the IHDP dataset was synthesized according to several different stochastic models on the \emph{observed} feature set. In this way, ignorability is guaranteed. Depending on the specific sample of the outcome model, i.e., whether variables correlated with race have strong influence or not, the lack of overlap varies in its impact on the results.

In our experiments, we use observations generated using setting ``A'' in the NPCI package~\citep{npci}, corresponding to response surface (outcome function) ``B'' in~\citet{hill2011bayesian}. This model follows an exponential-linear form for the outcome under treatment and a linear form for the controls, ensuring that their difference, CATE, is a nonlinear function. Sparsity in the coefficients is introduced through random sampling with a probability 0.6 that a coefficient is exactly equal to 0. The full description of the model may be found in \citet{hill2011bayesian}, Section~4.1. The specific realizations (draws) used in our evaluation can be accessed at~\url{http://www.mit.edu/~fredrikj/}.

\begin{table}[t!]
  \begin{center}
  \caption{\label{tbl:ihdp_results}Mean squared error, and standard error over 1000 random draws of the outcome model, in estimates of CATE and ATE on IHDP within-sample (left) and out-of-sample (right). Lower is better. $^\dagger$Not applicable. \vspace{1.0em}}
    \begin{tabular}{l|cc|cc}
  \toprule
  \multicolumn{1}{l}{} & \multicolumn{2}{c}{\bf{Within sample}} &
    \multicolumn{2}{c}{\bf{Out of sample}} \\
  \midrule
  & {\sc mse cate} & {\sc mse ate} &  {\sc mse cate} & {\sc mse ate} \\
  \midrule
  OLS-S & $5.8 \pm 0.3$ & $0.73 \pm 0.04$ & $5.8 \pm 0.3$ & $0.94 \pm 0.06$ \\
  OLS-T  & $2.4 \pm 0.1$ & $0.14 \pm 0.01$ & $2.5 \pm 0.1$ & $0.31 \pm 0.02$ \\

  BLR  & $5.8 \pm 0.3$ & $0.72 \pm 0.04$ & $5.8 \pm 0.3$ & $0.93 \pm 0.05$ \\
  $k$-NN & $2.1 \pm 0.1$ & $0.14 \pm 0.01$ & $4.1 \pm 0.2$ & $0.79 \pm 0.05$ \\
  TMLE  & $5.0 \pm 0.2$ & $0.30 \pm 0.01$ & $\dagger$ & $\dagger$ \\
  BART  & $2.1 \pm 0.1$ & $0.23 \pm 0.01$ & $2.3 \pm 0.1$ & $0.34 \pm 0.02$ \\
  R.For.  & $4.2 \pm 0.2$ & $0.73 \pm 0.05$ & $6.6 \pm 0.3$ & $0.96 \pm 0.06$ \\
  C.For.  & $3.8 \pm 0.2$ & $0.18 \pm 0.01$ & $3.8 \pm 0.2$ & $0.40 \pm 0.03$ \\
  BNN  & $2.2 \pm 0.1$ & $0.37 \pm 0.03$ & $2.1 \pm 0.1$ & $0.42 \pm 0.03$  \\
  \midrule
  TARNet  & $0.88 \pm 0.02$ & $0.26 \pm 0.01$ & $0.95 \pm 0.02$ & $0.28 \pm 0.01$ \\
  CFR$_{\mbox{MMD}}$ & $0.73 \pm 0.01$ & $0.30 \pm 0.01$ & $0.78 \pm 0.02$ & $0.31 \pm 0.01$ \\
  CFR$_{\mbox{Wass}}$ & $0.71 \pm 0.02$ & $0.25 \pm 0.01$ & $0.76 \pm 0.02$ & $0.27 \pm 0.01$ \\
  \bottomrule
\end{tabular}

  \end{center}
\end{table}

\begin{figure}[t!]
  \centering
  \begin{subfigure}[t]{.5\textwidth}
    \centering
    \includegraphics[height=2.2in]{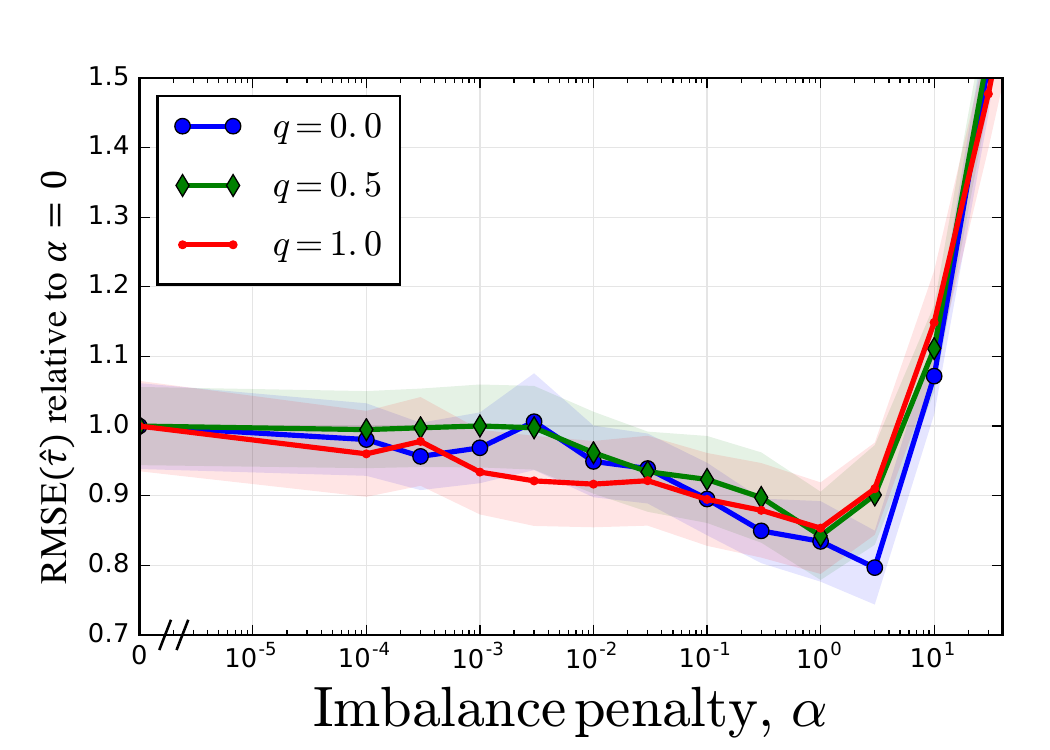}
    \caption{\label{fig:ihdp_vs_alpha}Ratio of mean squared squared error in estimated CATE relative to $\alpha=0$, as a function of the imbalance regularization $\alpha$ with uniform sample weights $w$, for different levels of introduced additional treatment group imbalance $q$. Uncertainty bands show standard errors over 500 realizations.}
  \end{subfigure}%
  \quad%
  \begin{subfigure}[t]{.46\textwidth}%
    \centering%
    \includegraphics[height=2.25in]{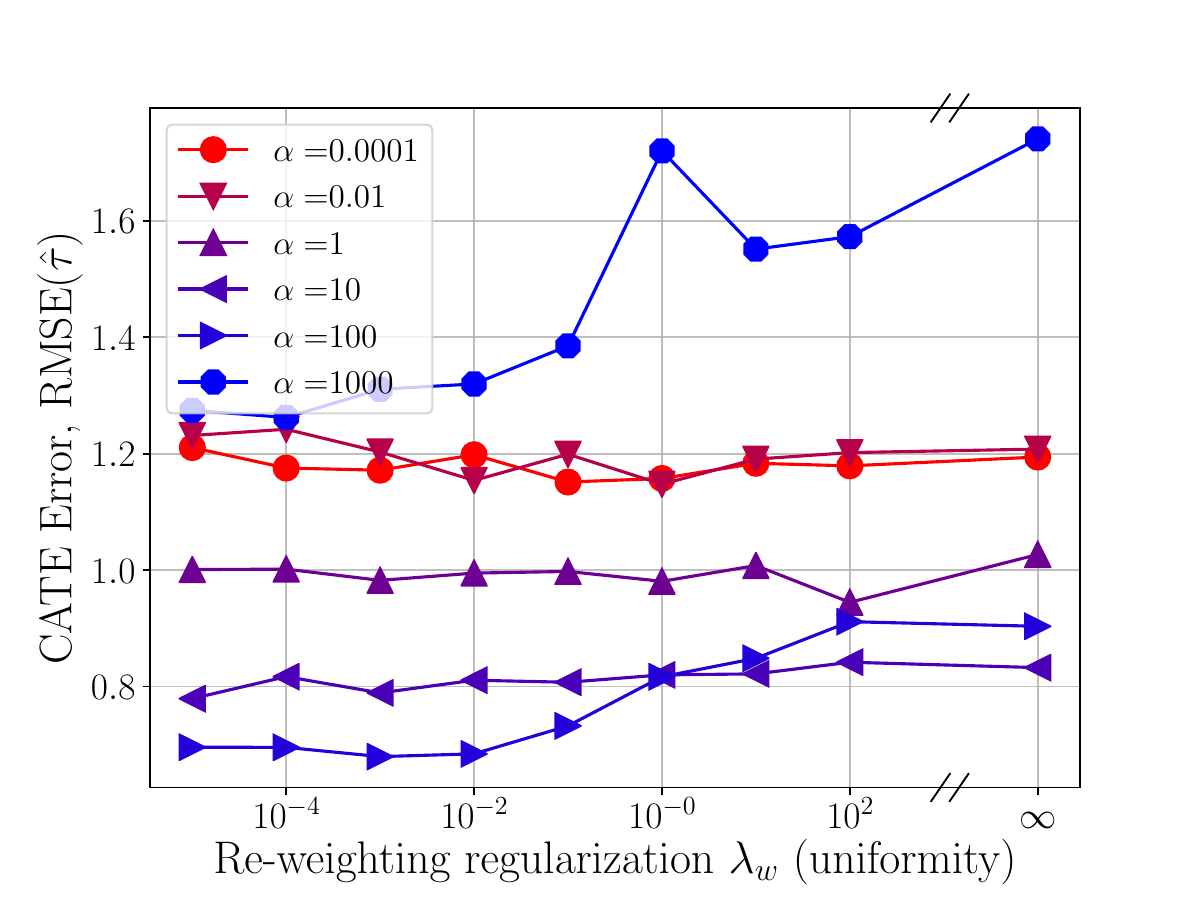}
    \caption{\label{fig:ihdp_weighting}RMSE for CFR estimates combining representation learning and re-weighting, minimizing \eqref{eq:emp_loss_w}, for varying weight regularization $\lambda_w$ and imbalance penalty $\alpha$. Higher $\lambda_w$ leads to more uniform weights.}
  \end{subfigure}
  \caption{\label{fig:ihdp}Results for estimating CATE on IHDP with different variants of the CFR model. In (a), we show results for the best performing architecture with uniform sample weights, varying the imbalance regularization $\alpha$. In (b), we show the results for a smaller architecture and their dependence on the uniformity of learned weights.}
\end{figure}

\paragraph{Results.}
The error in estimates of CATE on IHDP can be seen in Table~\ref{tbl:ihdp_results}. Here, we present only the variants of CFRwith uniform weighting, and refer to Figure~\ref{fig:ihdp_weighting} for a comparison between learned and uniform weights. First, we note that all of the proposed neural network estimators (TARNet and CFR variants) outperform the selected baselines. We attribute this, to a large extent, to multi-layer neural networks being a suitable function class for this dataset. CFR improves marginally over TARNet, indicating that regularizing distributional invariance is beneficial for prediction of CATE. We note also that, in general, the S-learner estimators (OLS-S and BNN) perform worse than separate or partially separate estimators (OLS-T, TARNet). The biggest differences between in-sample and out-of-sample performance are attained by the $k$-NN and random forest estimators.

\paragraph{Increasing imbalance.} In Figure~\ref{fig:ihdp_vs_alpha}, we study the effect of increasing the imbalance between treatment groups through biased subsampling. To do this, we fit a logistic regression propensity score model $\hat{p}(T=1\mid X=x)$ and for a parameter $q\geq 0$, we repeatedly remove the \emph{control} sample with largest estimated propensity with probability, $q$ and a random control observation with probability $1-q$, until 400 samples remain. For three values of $q$, we estimate CATE using CFR with uniform sample weights for different values of the penalty $\alpha$ of treatment group distance in the learned representation $\Phi$. We see that for small $\alpha$, as expected, the relative error is comparable to TARNet ($\alpha=0$), but that it decreases until $\alpha\approx 1$. For $\alpha > 2$, the performance deteriorates as the  influence of the input on the representation is constrained too heavily. As we'll see below, this may be partially remedied by sample weighting.

\paragraph{Learning the sample weights.} In Figure~\ref{fig:ihdp_weighting}, we study the quality of CFR estimates when sample weights are learned by minimizing objective \eqref{eq:emp_loss_w}. In this setting, the chosen model is intentionally restricted to have representations $\Phi$ of two layers with 32 and 16 hidden units each and hypotheses $h(\Phi)$ of a single layer with 16 units. This choice was made to emphasize the value of reweighting under model misspecification. The weighting function was modeled using two layers of 32 units each. We see in Figure~\ref{fig:ihdp_weighting} that a model using non-uniform sample weights ($\lambda_w$ small) is less sensitive to excessively large penalties $\alpha$. This is because the IPM term may now be minimized also by learning the weights, rather than only by constraining the capacity of $\Phi$. In the small-$\alpha$ regime, the non-uniformity of weights has almost no impact, as the incentive to reduce the IPM using the weights is too small. In this experiment, the best results are attained for combination of a considerably larger value of $\alpha$ and small penalty on the non-uniformity of weights. In general, we do not observe any adverse effects of having a small value of $\lambda_w$. This is likely due to the choice of architecture for the learned weighting function already constraining the weights.

%
%
\subsection{Partially randomized study: National Supported Work program.}
\citet{lalonde1986evaluating} carried out a widely known experimental study of the effect of job training on future income and employment status based on the National Supported Work (NSW) program. Later, \citet{smith2005does} combined the LaLonde study with observational data to form a larger dataset which has been used frequently as a benchmark in the causal inference community. The presence of the randomized subgroup allows for straightforward estimation of average treatment effects and policy value.

The original study by \citet{smith2005does} includes 8 covariates such as age and education, as well as previous earnings. The treatment indicates participation in the NSW job training program. By construction, all treated subjects belong to the LaLonde experimental cohort; the observational cohort includes only controls. Additionally, the nature of the observational cohort is such that overlap is minimal at best---the experimental cohort may be separated from the observational using a linear classifier with 96\% accuracy. This means that global estimators of the control outcome applied to the treated, such as linear models or difference-in-means estimators of causal effects are likely to suffer severe bias if not re-weighted.

Based on the original outcome measuring yearly earnings at the end of the study, we construct a binary classification task called Jobs, in which the goal is to predict unemployment. Following \citet{dehejia2002propensity}, we use an expanded feature set that introduces interaction terms between some of the covariates. The task is based on the cohort used by \citet{smith2005does} which combines the LaLonde experimental sample (297 treated, 425 control) and the ``PSID'' comparison group (2490 control). There were 482 (15\%) subjects unemployed by the end of the study. In our experiments, we average results over 10 train/validation/test splits of the full cohort with ratios 56/24/20. We train CFR methods with uniform weighting, according to \eqref{eq:emp_loss}, selecting the imbalance parameter $\alpha$ according to held-out policy risk.

\paragraph{Results.}
\begin{table}[t!]
  \begin{center}
  \caption{\label{tbl:jobs_results}Policy risk and mean squared error in estimates of ATT on Jobs within-sample (left) and out-of-sample (right). Lower is better. $^\dagger$Not applicable. \vspace{1.0em}}
    \begin{tabular}{l|cc|cc}
  \toprule
  \multicolumn{1}{l}{} & \multicolumn{2}{c}{\bf{Within sample}} &
    \multicolumn{2}{c}{\bf{Out of sample}} \\
  \midrule
  & $\hat{R}_{\mbox{\tiny Pol}}$ & {\sc mse att} &  $\hat{R}_{\mbox{\tiny Pol}}$ & {\sc mse att} \\
  \midrule
  LR-S & $0.22 \pm 0.00$ & $0.01 \pm 0.00$ & $0.23 \pm 0.02$ & $0.08 \pm 0.04$\\
  LR-T  & $0.21 \pm 0.00$ & $0.01 \pm 0.01$ & $0.24 \pm 0.01$ & $0.08 \pm 0.03$\\
  BLR  & $0.22 \pm 0.01$ & $0.01 \pm 0.01$ & $0.25 \pm  .02$ & $0.08 \pm 0.03$\\
  $k$-NN & $0.02 \pm 0.00$ & $0.21 \pm 0.01$ & $0.26 \pm 0.02$ & $0.13 \pm 0.05$ \\
  TMLE  & $0.22 \pm 0.00$ & $0.02 \pm 0.01$ & $\dagger$ & $\dagger$ \\
  BART  & $0.23 \pm 0.00$ & $0.02 \pm 0.00$ & $0.25 \pm 0.02$ & $0.08 \pm 0.03$\\
  R.For.  & $0.23 \pm 0.01$ & $0.03 \pm 0.01$ & $0.28 \pm 0.02$ & $0.09 \pm 0.04$ \\
  C.For.  & $0.19 \pm 0.00$ & $0.03 \pm 0.01$ & $0.20 \pm 0.02$ & $0.07 \pm 0.03$ \\
  BNN  & $0.20 \pm 0.01$ & $0.04 \pm 0.01$ & $0.24 \pm 0.02$ & $0.09 \pm 0.04$\\
  TARNet & $0.17 \pm 0.01$ & $0.05 \pm 0.02$ & $0.21 \pm 0.01$ & $0.11 \pm 0.04$\\
  CFR$_{\mbox{MMD}}$ & $0.18 \pm 0.00$ & $0.04 \pm 0.01$ & $0.21 \pm 0.01$ & $0.08 \pm 0.03$  \\
  CFR$_{\mbox{Wass}}$ & $0.17 \pm 0.01$ & $0.04 \pm 0.01$ & $0.21 \pm 0.01$ &  $0.09 \pm 0.03$ \\
  \bottomrule
\end{tabular}

  \end{center}
\end{table}
\begin{figure}[t!]
  \centering
  \includegraphics[width=.5\textwidth]{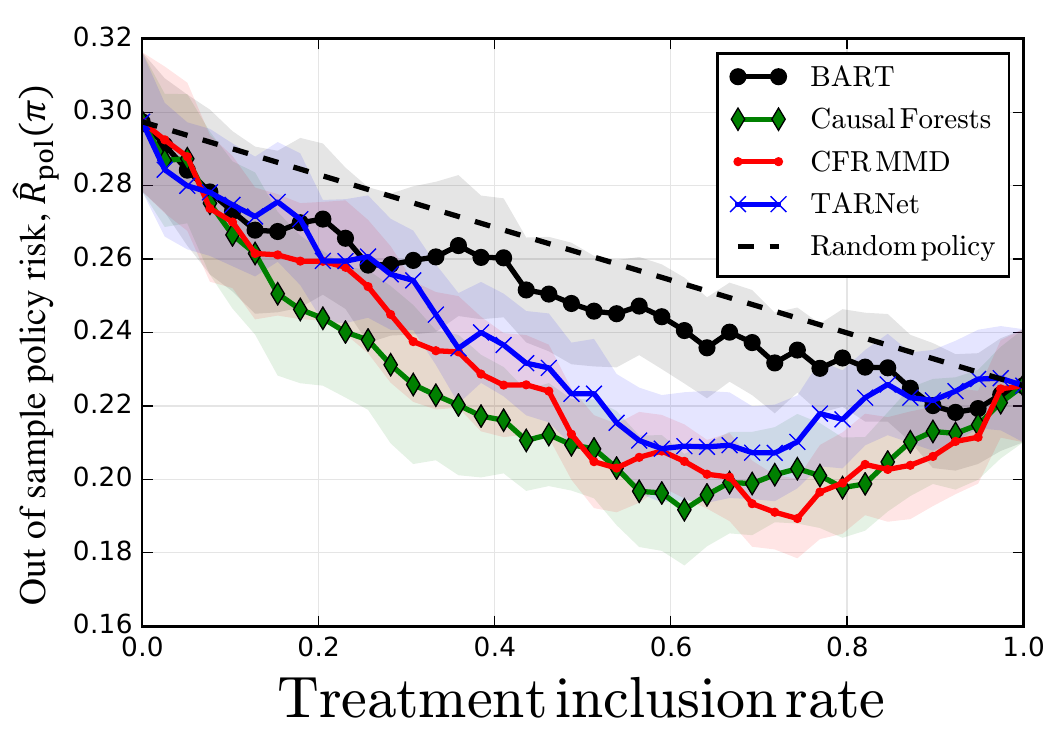}
  \caption{\label{fig:jobs_policy_curve}Policy risk as a function of treatment inclusion rate on Jobs. Lower is better. Subjects are included in treatment in order of their estimated treatment effect given by the various methods. CFR Wass is similar to CFR MMD and is omitted to avoid clutter}
\end{figure}
In Table~\ref{tbl:jobs_results}, we give the policy risk $\hat{R}_{\mbox{\tiny Pol}}$ evaluated over the randomized component of Jobs, as defined in \eqref{eq:policy_curve}, and mean squared error in the estimated average treatment effect on the treated ({\sc mse att}). The policy we consider in the table assigns treatment to the top subject for which the CATE is estimated to be positive ($\lambda=0$). We can see from the results that, despite the significant lack of overlap, the difference between linear and non-linear estimators is much less pronounced than for IHDP. This is likely partly due to the features used in the jobs dataset which have been handcrafted to predict the outcome of interest well. In contrast, the IHDP outcome is non-linear by construction.

We also see that straightforward logistic regression does remarkably well in estimating the ATT. However, being a linear model, logistic regression can only ascribe a uniform policy -- in this case, ``treat everyone''. The more nuanced policies offered by non-linear methods achieve lower policy risk, though this difference is less pronounced in the out of sample case, indicating that part of the difference may be due to overfitting. The nearest-neighbor estimator $k$-NN appears to perform incredibly well within-sample, but generalizes poorly to the hold-out. Additionally, its estimate of the ATT is the worst among the baselines.

In Figure~\ref{fig:jobs_policy_curve}, we plot policy risk as a function of treatment threshold $\lambda$, as defined in $\eqref{eq:lambda_pol_risk}$. This is described in the figure as varying the fraction of subjects treated in a policy that treats only the subjects with the largest estimated CATE. Overall, the benefits of imbalance regularization of the CFR models offer less advantage than on IHDP. This may be due to the smaller covariate set of Jobs containing less redundant features than those in IHDP. Recall that the IHDP outcome coefficients have 60\% sparsity in the feature set, by design. In contrast, the Jobs covariate set has been hand-picked to account for confounding. This means that one of the benefits of imbalance-regularizations of representations---to exclude variables only predictive of treatment---is likely to have a smaller effect in comparison.

%
%

\section{Discussion}
\label{sec:discussion}
We have presented generalization bounds for estimation of potential outcomes and causal effects from observational data. These bounds were used to derive learning objectives for estimation algorithms that proved successful in empirical evaluation. The bounds do not rely on the so-called treatment group overlap (or positivity) assumption, common to most studies of causal effects from observational data. This assumption states that for any one observed subject, there is some probability that they were prescribed either treatment option. Removing this assumption means that we cannot identify the causal effect non-parametrically but, as we show in this work, we can still bound the expected error (risk) of any hypothesis in a given class.

When can we expect overlap to not hold yet identification to be possible? One example is when many of the covariates in the conditioning set $\cX$ have a strong effect on treatment but only a weak or non-existent effect on the outcome. For example, if some of the covariates in $\cX$ are actually instrumental variables, conditioning on them might substantially increase variance and reduce overlap, with no gain in estimating the CATE function \citep{brookhart2010confounding,shortreed2017outcome}. We conjecture that this might often be the case in high-dimensional cases: in aggregate, there might not be nominal overlap with respect to the measured covariates, while at the same time many of them are actually only weak confounders, or even not confounders at all; see also \citet{d2017overlap}.

Our results offer several new perspectives on causal effect estimation. In particular, they bring together two hitherto separate approaches to dealing with treatment group shift---representation learning and sample re-weighting---and give insight as to when either approach is likely to be more successful, and when should they be used together. It is well known that under the overlap and ignorability assumptions, ordinary risk minimization leads to consistent estimation of causal effects~\citep{pearl2009causality,ben2012hardness,alaa2018limits} in the limit of infinite samples, but the hardness of the problem is less well understood in the finite sample case, or when overlap is violated. Our results provide some insight in this setting.

{\blockedit
The bounds in this work are limited to binary treatment choices. A natural question to ask is how to generalize our results to settings with multiple discrete treatments or continuous ones. In the former case, our bounds on the risk in predicting potential outcomes can be trivially extended using a one-vs-all approach, evaluating the factual risk for subjects on a particular treatment and the counterfactual risk of subjects on any other treatment. However, the bounds for causal effects do not extend in the same way to this approach since they are aimed at evaluating the difference between two potential outcomes. An alternative for this purpose is to consider measuring the all-pairs distributional difference between treatment groups, but this is unlikely to be efficient in practice for large numbers of treatments. The right approach will depend on the causal estimand of interest. In real-world applications, comparing the value of treatment policies over a large number of treatments may be more important than comparing the effects of each pair of treatments. To this end, \citet{liu2018representation} applied ideas put forth in this work to off-policy policy evaluation. 
}

{\blockedit It is customary in machine learning to evaluate methodological progress based on performance on a small number of well-established benchmarks, such as MNIST~\citep{lecun1998gradient} or ImageNet~\citep{deng2009imagenet}. Similarly, IHDP has become a de facto benchmark for causal effect estimation~\citep{hill2011bayesian,shalit2016estimating,alaa2018limits,shi2019adapting,curth2021doing}. However, IHDP is smaller than most machine learning benchmarks and susceptible to ``test set overfitting''. Even disregarding the size issue, it may be argued that benchmarks for causal effect estimation are more prone to going stale as the strong assumptions we make (or synthesize) need not hold in the tasks we wish to apply our models to \citep{curth2021doing}. Moreover, the relatively simple form of the outcome model, the small dimensionality, and the structured fully observed nature of the data makes IHDP a much easier challenge than what we may face in for example analysis of electronic healthcare records. The Jobs dataset, on the other hand, offers realistic data but suffers from small sample size and a small number of measured covariates. Towards understanding the behavior of different estimators, datasets like IHDP and Jobs are not necessarily representative samples of the problems we may encounter in applications. However, we believe that the least a proposed method should show is that under controlled, sometimes idealized conditions, it performs on par with similar methods \citep{dorie2019rejoinder}. Indeed, as \citet{curth2021doing} advocate, we can see that among a class of similar neural network models our proposed methods perform well, and the addition of balancing terms improved performance; we might however set only limited stock in comparisons with methods which makes completely different assumptions about the data, such as Causal Forest \citep{wager2018estimation}.
Finally, we believe it is of utmost importance for the field as a whole to produce a larger set of benchmarks that reflect the diversity of real-world observational studies, and that the recent ACIC challenge~\citep{shimoni2018benchmarking} is a good step in this direction.  }

%
%

\subsection*{Acknowledgments}
We thank Ahmed Alaa, Shira Mitchell, Rajesh Ranganath, Alexander D’Amour, Zach Lipton, Jennifer Hill, Rahul Krishnan, Michael Oberst, Hunter Lang and  Christina X Ji for insightful feedback and discussions. The preparation of this manuscript was supported in part by Office of Naval Research Award Nos. N00014-17-1-2791 and N00014-21-1-2807, the MIT-IBM Watson AI Lab and the Wallenberg AI, Autonomous Systems and Software Program (WASP) funded by the Knut and Alice Wallenberg Foundation.

\vskip 0.2in
\bibliography{cate_jmlr}

\clearpage

%
%

\appendix

%
%
\allowdisplaybreaks
\section{Proof of Theorem~\ref{thm:asymptotics}}
\label{app:asymptotics}
{\blockedit
We prove Theorem~\ref{thm:asymptotics}, first repeating the original statement below. Recall that $\cO_t(h_t, \Phi, w^t; \beta)$ is the weighted learning objective \eqref{eq:emp_loss_w} and that $R(f_t)$ is the marginal risk for predicting a single potential outcome for the whole population, as in \eqref{eq:marg_risk}, 
$$
R(f_t) = \E_{X,Y}[L(Y(t), f_t(x))]~.
$$

%
%
\begin{mdframed}[innerbottommargin=.8em,innertopmargin=0em]%
\begin{reptheorem}{thm:asymptotics}
Suppose $\cH$ is a reproducing kernel Hilbert space (RKHS) given
by a bounded kernel $k$, such that for all $h_t\in \cG, \Phi\in \cE$, $h_t \circ \Phi \in \cH$.
Suppose weak overlap holds in that 
$$\forall t \in \{0,1\} : \E_X[(p_t(X)/p_{1-t}(X))^2] < \infty~.
$$
Then, with $\cO_t$ the objective defined in \eqref{eq:emp_loss_w}, \edit{with the IPM-space $\cL$ also an RKHS with bounded kernel $k_\cL$}, and $n_t = \sum_{i=1}^n\mathds{1}[t_i = t]$ for $t\in \{0,1\}$,
$$
\min_{h_t,\Phi,w^t}\cO_t(h_t, \Phi, w^t; \beta) \leq \min_{f_t \in \cH}R(f_t) + O_p(1/\sqrt{n_0}+1/\sqrt{n_1}) ~,
$$
\edit{where $O_p$ denotes stochastic boundedness.}
Thus, under the assumptions of Thm.~\ref{thm:main_rep}, for sufficiently large $\alpha$ and $\lambda_w$, with $\hat{f}^n$ the minimizer of \eqref{eq:emp_loss_w} for $n$ samples,
$$
R(\hat f_t^n)\leq \min_{f_t \in \cH}R(f_t) + O_p(1/n_0^{3/8}+1/n_1^{3/8}).
$$
In words, the minimizers of $\eqref{eq:emp_loss_w}$ converge to the representation and hypothesis that minimize the counterfactual risk, in the limit of infinite samples.
\end{reptheorem}
\end{mdframed}
}

{\blockedit

\begin{proof}
Let $f_t^*=\Phi^*\circ h_t^*\in\argmin_{f_t\in\cH}R(f_t)$
and let $w^{t,*}(x) \coloneqq p_{\Phi, 1-t}(\Phi^*(x))/p_{\Phi, t}(\Phi^*(x)) = p_{1-t}(x)/p_{t}(x)$ be the importance ratio between densities of treatment group $1-t$ and $t$. The second equality follows since $\Phi$ is assumed invertible.
Since $\min_{h_t,\Phi,w^t}\cO_t(h_t, \Phi, w^t; \beta)\leq \cO_t(h_t^*, \Phi^*, w^{t,*}; \beta)$,
it suffices to show that\footnote{Note the distinction between the objective $\cO$ and the $O$ of ``Big $O$'' notation. }
$$
\cO_t(h_t^*, \Phi^*, w^{t,*}; \beta) =  R(f_t^*) + O_p(1/\sqrt{n_0}+1/\sqrt{n_1})~,
$$ 
where $O_p$ denotes stochastic boundedness, defined as follows. $X_n = O_p(a_n)$ holds, i.e., $X_n/a_n$ is stochastically bounded, if there exist finite constants $M>0, N>0$, such that for every $\epsilon > 0$
$$
p(|X_n / a_n| > M) < \epsilon, \forall n > N~.
$$

We will work term by term with the learning objective \eqref{eq:emp_loss_w}, with $\beta=(\alpha, \lambda_w, \lambda_h)$, $\pi_t = p(T=t)$, and $\tilde{w}^t_i = \pi_t + \pi_{1-t} w^t_i$, 
\begin{align*}
& \cO_t(h_t, \Phi, w^t; \beta) = \underbrace{\sum_{i:t_i=t} \frac{\tilde{w}^t_i}{n_t} L(h_t(\Phi(x_i)), y_i)}_{\encircle{A}}  
+ \ \underbrace{\frac{\cR(h_t)}{\lambda_h^{-1} \sqrt{n_t}}}_{\encircle{B}} + \alpha\pi_{1-t} \underbrace{\ipm_{\cL}(\hat{p}_{\Phi, t}^{w^t}, \hat{p}_{\Phi, 1-t}^{\vphantom{w^t}})}_{\encircle{C}} + \underbrace{\frac{\|w^t\|_2}{\lambda^{-1}_{w} n_t}}_{\encircle{D}}.
\end{align*}
for the case $h_t \circ \Phi = f^*, w^t = w^{t,*}$.
For term $\encircle{D}$,
letting $w_i^{t,*}=w^{t,*}(x_i)$,
we have that by weak overlap
$$
\encircle{D}^2=\frac{1}{n_t}\times\frac{1}{n_t}\sum_{i:t_i=t}(w_i^{t,*})^2=O_p(1/n_t),
$$
so that $\encircle{D}=O_p(1/\sqrt{n})$. 

For term $\encircle{A}$, under ignorability, each term in the sum in the first term has expectation equal to the population marginal risk $R(f_t^*)$ over $p(X)$, since $p_t(x_i) \tilde{w}^{t,*}_i = p_t(x_i)(\pi_t + \pi_{1-t} w^{t,*}_i) = p(X=x_i)$. As a result, by weak overlap and bounded second moments of loss, we have $\encircle{A}=R(f_t^*)+O_p(1/\sqrt{n})$. For term $\encircle{B}$, since $h_t^*$ is fixed, we have deterministically that $\encircle{B}=O(1/\sqrt{n_t})$.

Finally, we address term $\encircle{C}$, which when expanded can be written as
$$\sup_{\|\ell\|_\cL \leq1} \left|\frac{1}{n_{1-t}} \sum_{i:t_i \neq t}\ell(\Phi^*(x_i))-\frac{1}{n_t} \sum_{i:t_i = t} w_i^{t,*} \ell(\Phi^*(x_i)) \right|.$$
Let $x'_i$ for $i=1,\dots, n_{1-t}$ and $x''_i$ for $i=1, \dots, n_t$ replicates of treatment group $1-t$, i.e., new ghost samples, drawn  from the treatment group opposite to $t$. Recall that $n=n_0+n_1$. Without loss of generality, assume that $t=1$ and that samples are ordered such that $t_i=0$ for $i\in \{1, ..., n_0\}$ and $t_i=1$ for $i\in \{n_0+1, ..., n\}$. Ghost samples are therefore drawn from $p_0(X)$. The case for $t=0$ follows immediately.  By Jensen's inequality,

\begin{align*}
\bbE[\encircle{C}^2] & =
\bbE\left[\sup_{\|\ell\|_\cL \leq 1} \left(\frac{1}{n_0} \sum_{i=1}^{n_0} \ell(\Phi^*(x_i)) - \frac{1}{n_1}\sum_{i=n_0+1}^{n} w_i^{1,*}\ell(\Phi^*(x_i) \right)^2 \right] \\
& =
\bbE\bigg[ \sup_{\|\ell\|_\cL  \leq 1} \bigg(\frac{1}{n_0} \sum_{i=1}^{n_0} (\ell(\Phi^*(x'_i)) - \bbE[\ell(\Phi^*(x'_i))]) \\
& \phantom{=} - \frac{1}{n_1} \sum_{i=n_0+1}^{n} (w_i^*\ell(\Phi^*(x_i)) - \bbE[\ell(\Phi^*(x''_i))]) \bigg)^2 \bigg] \\
& \leq
\bbE\bigg[ \sup_{\|\ell\|_\cL \leq 1}\bigg( \frac{1}{n_0} \sum_{i=1}^{n_0} (\ell(\Phi^*(x_i))-\ell(\Phi^*(x'_i))) \\
& \phantom{=} -\frac{1}{n_1} \sum_{i=n_0+1}^{n} (w_i^{1,*}\ell(\Phi^*(x_i))-\ell(\Phi^*(x''_i))) \bigg)^2 \bigg] \\
& \leq
2\bbE\bigg[ \sup_{\|\ell\|_\cL \leq1}\bigg( \frac{1}{n_0} \sum_{i=1}^{n_0} (\ell(\Phi^*(x_i)) - \ell(\Phi^*(x'_i))) \bigg)^2 \bigg] \\
& \phantom{=} + 2\bbE[\sup_{\|\ell\|_\cL \leq1}(\frac{1}{n_1} \sum_{i=n_0+1}^{n} (w_i^{1,*}\ell(\Phi^*(x_i))-\ell(\Phi^*(x''_i))) \bigg)^2 \bigg]
\end{align*}

Let $\xi_i(\ell)=\ell(\Phi^*(x_i))-\ell(\Phi^*(x'_i))$ and let $\zeta_i(h)=w_i^{1,*}\ell(\Phi^*(x_i))-\ell(\Phi^*(x''_i))$. Note that for every $\ell$, $\bbE[\zeta_i(\ell)]=\bbE[\xi_i(\ell)]=0.$ Moreover, \edit{since the RKHS $\cL$ has bounded kernel by assumption}, 
$$
\bbE[\|\zeta_i\|^2] \leq 4\E[k_\cL(\Phi^*(x'_i),\Phi^*(x'_i))]\leq M
$$
for some $M$. Similarly, $\bbE[\|\xi_i\|^2]\leq 2\E[(w_i^{1,*})^2]M+2M\leq M'<\infty$ because of weak overlap. Let $\zeta_i'$ for $i=1,\dots,n$ be i.i.d. replicates of $\zeta_i$ (ghost sample) and let $\epsilon_i$ be i.i.d. Rademacher random variables. Because $\cL$ is a Hilbert space, we have that $\sup_{\|\ell\|_\cL \leq 1}(A(\ell))^2=\|A\|^2=\left<A,A\right>$. Therefore, by Jensen's inequality,
\begin{align*}
& \bbE \bigg[ \sup_{\|\ell\|_\cL \leq1}\bigg( \frac{1}{n_1} \sum_{i=n_0+1}^n(w_i^{1,*}\ell(\Phi^*(x_i))-\ell(\Phi^*(x''_i))) \bigg)^2 \bigg] \\
& = \bbE \bigg[ \sup_{\|\ell\|_\cL \leq1}\bigg( \frac{1}{n_1} \sum_{i=n_0+1}^n\zeta_i(\ell) \bigg)^2 \bigg] \\
& = \bbE \bigg[ \sup_{\|\ell\|_\cL \leq1}\bigg( \frac{1}{n_1} \sum_{i=n_0+1}^n(\zeta_i(\ell)-\bbE[\zeta'_i(\ell)]) \bigg)^2 \bigg] \\
& \leq \bbE \bigg[ \sup_{\|\ell\|_\cL \leq1}\bigg( \frac{1}{n_1} \sum_{i=n_0+1}^n(\zeta_i(\ell)-\zeta'_i(\ell)) \bigg)^2 \bigg] \\
& = \bbE \bigg[ \sup_{\|\ell\|_\cL \leq1}\bigg( \frac{1}{n_1} \sum_{i=n_0+1}^n\epsilon_i(\zeta_i(\ell)-\zeta'_i(\ell)) \bigg)^2 \bigg] \\
& \leq \frac4{n_1^2}\bbE \bigg[ \sup_{\|\ell\|_\cL \leq1}\bigg( \sum_{i=n_0+1}^n\epsilon_i\zeta_i(\ell) \bigg)^2 \bigg] \\
& = \frac4{n_1^2}\bbE \bigg[ \|\sum_{i=n_0}^n\epsilon_i\zeta_i\|^2 \bigg] \\
& = \frac4{n_1^2}\bbE\bigg[ \sum_{i,j=n_0+1}^n\epsilon_i\epsilon_j\left<\zeta_i,\zeta_j\right> \bigg] \\
& = \frac4{n_1^2}\bbE[\sum_{i=n_0+1}^n\|\zeta_i\|^2] \\
& = \frac4{n_1^2}\sum_{i=n_0+1}^n \bbE[\|\zeta_i\|^2] \\
& \leq \frac{4M'}{n_1}
\end{align*}
An analogous argument can be made of $\xi_i$'s, showing that $\bbE[\encircle{C}^2]=O(1/n)$ and hence $\encircle{C}=O(1/\sqrt{n})$ by Markov's inequality. The final result follows from Theorem~\ref{thm:main_rep}.
\end{proof}

}

{\blockedit
We may generalize Theorem~\ref{thm:asymptotics} to bound the risk in predicting the outcome \edit{$Y = Y(T)$} in expectation over a treatment policy $p_\pi(T\mid X)$ based on observations from a policy $p_\mu(T\mid X)$. The risk in predicting a single fixed potential outcome $t$ follows as a special \edit{(deterministic)} case of $p_\pi(T=t \mid X) = 1$. With this in mind, let

$$
R_\pi(f) = \E_X[\E_{T\mid X \sim p_\pi}[\ell_{f}(X, T)]]~,
$$

where $\ell_{f}(x, t) = \E[L(f(x, t), Y(t)) \mid X=x, T=t]$.
As previously, we consider hypotheses $f(x, t) = h(\Phi(x), t)$ for functions $h\in \cF$ and embeddings $\Phi \in \cE \subset \{ \cX \rightarrow \cZ\}$ for some embedding space $\cZ \subset \bbR^v$.

\begin{thmcol}\label{col:asymptotics_pol}
Suppose $\mathcal H$ is a reproducing kernel Hilbert space given by a bounded kernel. Further suppose that weak overlap holds in that
$\mathbb E[(p_\pi(x,t)/p_\mu(x,t))^2] < \infty$. Assume that $n$ labeled samples $\{(x_i, t_i, y_i)\}_{i=1}^n \sim p_\mu$ and $m$ unlabeled samples $\{(x_i, t_i)\}_{i=n+1}^{m+n} \sim p_\pi$ are available.
Then,
$$
\min_{h,\Phi,w}\cO_\pi(h, \Phi, w; \beta) \leq \min_{f\in\cF}R_\pi(f) + O(1/\sqrt{n}+1/\sqrt{m}) ~.
$$
\end{thmcol}
Corollary~\ref{col:asymptotics_pol} may be proven using an anologous argument to that of Theorem~\ref{thm:asymptotics}.
}

\section{Experiment details}
\label{app:exp}
See Table~\ref{tbl:hypparams} for a description of hyperparameters and search ranges for TARNet, CFR Wass and CFR MMD.

\begin{table}[t!]
  \caption{\label{tbl:hypparams}Hyperparameters and ranges.}
  \begin{center}
      \begin{tabular}{ll}
        Parameter & Range \\ \hline
        Imbalance parameter, $\alpha$ & $\{10^{k/2}\}_{k=-10}^6$ \\
        Num. of representation layers & $\{1,2,3\}$ \\
        Num. of hypothesis layers & $\{1,2,3\}$ \\
        Dim. of representation layers & $\{20, 50, 100, 200\}$ \\
        Dim. of hypothesis layers & $\{20, 50, 100, 200\}$ \\
        Batch size & $\{100, 200, 500, 700\}$ \\
        \hline
      \end{tabular}
    \end{center}
\end{table}

\section{Architecture for joint learning of sample weights}
\label{app:rcfr}
For an illustration of the re-weighed CFR estimator, see Figure~\ref{fig:rcfr}.
\begin{figure}
  \centering
  \includegraphics[height=.3\textwidth]{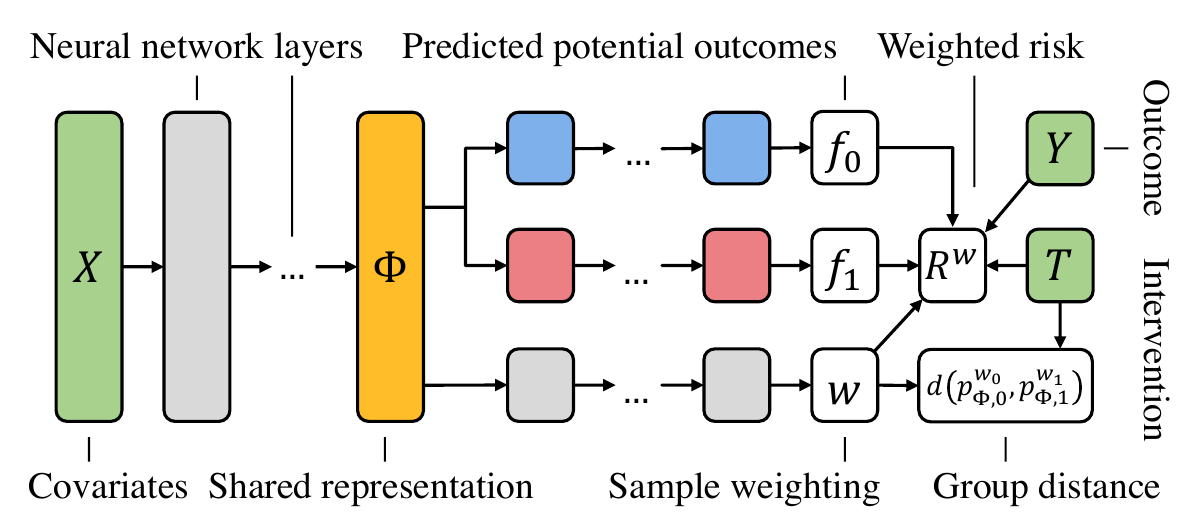}
  \caption{\label{fig:rcfr}Illustration of the Re-weighted Counterfactual Regression (RCFR) estimator. Green boxes indicate inputs, white boxes outputs and loss terms, yellow boxes shared representations and blue/red boxes estimators of potential outcomes. Solid lines indicate transformation part of the pr
  ediction function and dashed lines indicate computations part of the learning procedure.}
\end{figure}

%
%
\section{Minimization of approximate Wasserstein distances}
\label{app:wasserstein}
Computing (and minimizing) the Wasserstein distance traditionally involves solving a linear program, which may be prohibitively expensive for many practical applications. \citet{cuturi2013sinkhorn} showed that introducing entropic regularization in the optimization problem results in an approximation computable through the Sinkhorn-Knopp matrix scaling algorithm, at orders of magnitude faster speed. The approximation, called Sinkhorn distances, is computed using a fixed-point iteration involving repeated multiplication with a kernel matrix $K$. We use the algorithm of \citet{cuturi2013sinkhorn} in our framework by differentiating through the iterations. See Algorithm~\ref{alg:wassgrad} for an overview of how to compute the gradient $g_1$ in Algorithm~\ref{eq:emp_loss}. When computing $g_1$, disregarding the gradient $\nabla_{\bf W} T^*$ amounts to minimizing an upper bound on the Sinkhorn transport. More advanced ideas for stochastic optimization of this distance have recently proposed by \citet{aude2016stochastic}, and might be used in future work.

\begin{algorithm}[tbp]
\caption{Computing the stochastic gradient of the Wasserstein distance}
\label{alg:wassgrad}
\begin{algorithmic}[1]
  \STATE \textbf{Input:} Factual  $(x_1,t_1,y_1), \ldots , (x_n,t_n,y_n)$, representation network $\Phi_{\bf{W}}$ with current weights by $\bf{W}$
  \STATE Randomly sample a mini-batch with $m$ treated and $m'$ control units $(x_{i_1},0,y_{i_1}), \ldots , $\\
  $(x_{i_m},0,y_{i_m}),  (x_{i_{m+1}},1,y_{i_{m+1}}), \ldots , (x_{i_{2m}},1,y_{i_{2m}}) $
   \STATE Calculate the $m \times m$ pairwise distance matrix between all treatment and control pairs $M(\Phi_{\bf{W}})$: \\
   $M_{kl}(\Phi) = \|\Phi_{\bf{W}}(x_{i_k}) - \Phi_{\bf{W}}(x_{i_{m+l}})\|$
    \STATE Calculate the approximate optimal transport matrix $T^*$ using Algorithm 3 of \citet{cuturi2014fast}, with input $M(\Phi_{\bf{W}})$
  \STATE Calculate the gradient:\\ $g_1 = \nabla_{\bf{W}} \left< T^*,M(\Phi_{\bf{W}})\right>$
  \vspace{0.2em}
\end{algorithmic}
\end{algorithm}

While our framework is agnostic to the parameterization of $\Phi$, our experiments focus on the case where $\Phi$ is a neural network. For convenience of implementation, we may represent the fixed-point iterations of the Sinkhorn algorithm as a recurrent neural network, where the states $u_t$ evolve according to
$$
u_{t+1} = n_t ./ (n_c K (1./(u_t^\top K)^\top))~.
$$
Here, $K$ is a kernel matrix corresponding to a metric such as the euclidean distance, $K_{ij} = e^{-\lambda\|\Phi(x_i) - \Phi(x_j)\|_2}$, and $n_c, n_t$ are the sizes of the control and treatment groups. In this way, we can minimize our entire objective with most of the frameworks commonly used for training neural networks, out of the box.

\end{document}